\documentclass{article}

\usepackage{microtype}
\usepackage{graphicx}
\usepackage{subfigure}
\usepackage{booktabs} 

\usepackage{hyperref}

\usepackage[accepted]{icml2024}

\usepackage{amsmath}
\usepackage{amssymb}
\usepackage{mathtools}
\usepackage{amsthm}
\usepackage{amsfonts}
\usepackage{mathrsfs}
\usepackage{multirow}
\usepackage{booktabs}
\usepackage{float}
\usepackage{bm}
\usepackage{colortbl}
\usepackage{enumitem}
\usepackage[T1]{fontenc}
\usepackage{aecompl}
\usepackage{dsfont}
\usepackage{enumitem}
\usepackage{subfigure}
\usepackage{balance}

\usepackage[capitalize,noabbrev]{cleveref}

\newcommand{\norm}[1]{\left\lVert#1\right\rVert}

\theoremstyle{plain}
\newtheorem{theorem}{Theorem}[section]
\newtheorem{proposition}[theorem]{Proposition}
\newtheorem{lemma}[theorem]{Lemma}

\theoremstyle{definition}
\newtheorem{definition}[theorem]{Definition}
\newtheorem{assumption}[theorem]{Assumption}
\theoremstyle{remark}

\usepackage[textsize=tiny]{todonotes}

\icmltitlerunning{Towards Certified Unlearning for Deep Neural Networks}

\begin{document}

\twocolumn[
\icmltitle{Towards Certified Unlearning for Deep Neural Networks}




\begin{icmlauthorlist}
\icmlauthor{Binchi Zhang}{uva}
\icmlauthor{Yushun Dong}{uva}
\icmlauthor{Tianhao Wang}{uva}
\icmlauthor{Jundong Li}{uva}
\end{icmlauthorlist}

\icmlaffiliation{uva}{University of Virginia, Charlottesville, VA, USA}

\icmlcorrespondingauthor{Jundong Li}{jundong@virginia.edu}

\icmlkeywords{Machine Unlearning, Certified Unlearning, Deep Neural Networks}

\vskip 0.3in
]



\printAffiliationsAndNotice{}  

\begin{abstract}
In the field of machine unlearning, certified unlearning has been extensively studied in convex machine learning models due to its high efficiency and strong theoretical guarantees.
However, its application to deep neural networks (DNNs), known for their highly nonconvex nature, still poses challenges.
To bridge the gap between certified unlearning and DNNs, 
we propose several simple techniques to extend certified unlearning methods to nonconvex objectives.
To reduce the time complexity, we develop an efficient computation method by inverse Hessian approximation without compromising certification guarantees.
In addition, we extend our discussion of certification to nonconvergence training and sequential unlearning, considering that real-world users can send unlearning requests at different time points.
Extensive experiments on three real-world datasets demonstrate the efficacy of our method and the advantages of certified unlearning in DNNs.
\end{abstract}

\section{Introduction}
In modern machine learning applications, model training often requires access to vast amounts of user data. 
For example, computer vision models often rely on images sourced from platforms such as Flickr, shared by its users~\citep{thomee2016yfcc100m}. 
However, as concerns about user data privacy grow, recent legislation such as the General Data Protection Regulation~\citep{GDPR} and the California Consumer Privacy Act~\citep{CCPA} have underscored the importance of \textit{the right to be forgotten}, which enables users to request the deletion of their data from entities that store it. 
The emergence of \textit{the right to be forgotten} has prompted the development of a new research domain known as machine unlearning~\citep{sommer2022athena,mercuri2022introduction,nguyen2022survey,xu2023machine,zhang2023review}, i.e., deleting user data and its lineage from a machine learning model.
Although the target data can be effectively unlearned by retraining the model from scratch, this method can incur a high computational cost, limiting its practical applications.
To address this problem, certified unlearning was proposed to find an efficient approximation of the retrained model.
In particular, the difference between the distributions of the unlearned model and the retrained model is bounded, making certified unlearning a stringent unlearning notion.

Despite the extensive study of certified unlearning in various machine learning models such as linear models~\citep{guo2020certified,izzo2021approximate,mahadevan2021certifiable}, general convex models~\citep{ullah2021machine,sekhari2021remember,neel2021descent}, Bayesian models~\citep{nguyen2020variational}, sum-product networks~\citep{becker2022certified}, and graph neural networks~\citep{pan2023unlearning,chien2023efficient,wu2023certified,wu2023gif,dong2024idea}, the understanding of certified unlearning for \textit{deep neural networks} (DNNs) remains nascent.
Most of the previous unlearning methods applied to DNNs~\citep{golatkar2020eternal,golatkar2021mixed,tarun2023fast,chundawat2023can} only introduced empirical approximation of the retrained model and failed to obtain a theoretical certification.
Although \citeauthor{mehta2022deep} tentatively analyzed the certification for DNNs, their results strongly rely on the convexity assumption, which is unrealistic and yields a significant gap between certified unlearning and DNNs.

To bridge the gap between certified unlearning and DNNs, we investigate the problem of certified unlearning without relying on convexity assumptions.
First, we divide certified unlearning into two steps, estimating the retrained model with a bounded error and adding random noises corresponding to the approximation error bound.
Existing works~\citep{golatkar2020eternal,golatkar2021mixed,mehta2022deep,warnecke2023machine} typically estimate the retrained model based on a single-step Newton update for convex objectives.
In this paper, we demonstrate that by making simple modifications to the Newton update, we can establish an approximation error bound for nonconvex (more precisely, weakly convex~\citep{davis2019proximally,davis2019stochastic}) functions such as DNNs.
It is worth noting that our proposed techniques can be adapted to different approximate unlearning strategies for convex models to improve their soundness in DNNs.
To improve the efficiency of the Newton update, we follow previous work to exploit a computationally efficient method to estimate the inverse Hessian in the Newton update and prove that the efficient method preserves a bounded approximation error.
In addition, we explore the flexibility of our method to complex real-world scenarios in two specific cases where strategies such as early stopping can stop the training from convergence and users can send unlearning requests at different time points.
Consequently, we adapt our certified unlearning approach to the nonconvergence training (the trained model is not a local optimum) and the sequential manner (the current unlearning process stems from a formerly unlearned model) without compromising certification guarantees.
Finally, we conduct extensive experiments including ablation studies to verify the effectiveness of our certified unlearning for DNNs in practice.
In particular, we use different metrics such as membership inference attacks~\citep{chen2021machine} and relearn time~\citep{golatkar2020eternal} to evaluate the unlearning performance in preserving the privacy of unlearned samples, providing empirical evidence of the effectiveness of our approach.
Further analysis explores the advantages of certified unlearning, e.g., stringency (certified unlearning outperforms other unlearning baselines with a relatively loose certification budget), efficiency, and robustness under sequential unlearning.

\section{Certified Unlearning}

\noindent\textbf{\textit{Preliminary.}}
Let $\mathcal{D}$ be a training dataset with $n$ data samples derived from the sample space $\mathcal{Z}$, and let $\mathcal{H}$ be the parameter space of a hypothesis class.
Let $\mathcal{A}:\mathcal{Z}^n\rightarrow\mathcal{H}$ be a randomized learning process that takes the training set $\mathcal{D}$ as input and outputs the optimal $\bm{w}^*\in\mathcal{H}$ that minimizes the empirical risk on $\mathcal{D}$ as
\begin{equation}
\bm{w}^*=\mathcal{A}(\mathcal{D})=\mathrm{argmin}_{\bm{w}\in\mathcal{H}}\mathcal{L}(\bm{w},\mathcal{D}).
\end{equation}
$\mathcal{L}$ is the empirical risk function that measures the error of a model over each training sample as $\mathcal{L}(\bm{w},\mathcal{D})=\frac{1}{n}\sum_{\bm{x}\in\mathcal{D}}l(\bm{w},\bm{x})$, where $l(\bm{w},\bm{x})$ is the loss function of the model $\bm{w}$ on the sample $\bm{x}$, e.g., cross-entropy loss and mean squared error. 
After the learning process, we obtain a trained model $\bm{w}^*=\mathcal{A}(\mathcal{D})$. 

In machine unlearning, to remove some training samples $\mathcal{D}_u\subset\mathcal{D}$ from the trained model $\bm{w}^*$, a common way is to leverage a randomized unlearning process $\mathcal{U}$ to update $\bm{w}^*$ and obtain an unlearned model $\bm{w}^-=\mathcal{U}(\bm{w}^*,\mathcal{D}_u,\mathcal{D})$, consequently.
We denote the number of unlearned samples as $n_u=|\mathcal{D}_u|$.
A straightforward but exact unlearning process is to retrain the model from scratch. 
Let $\tilde{\bm{w}}^*$ be the retrained model and $\mathcal{D}_r=\mathcal{D}\backslash\mathcal{D}_u$, then we have $\tilde{\bm{w}}^*=\mathrm{argmin}_{\bm{w}\in\mathcal{H}}\mathcal{L}(\bm{w},\mathcal{D}_r)$.
Despite the rigor of retraining, the high computational cost makes it a trivial method that cannot be used in practice.
Instead, certified unlearning processes find an unlearned model similar to the re-trained model, but with much less computation cost.
Next, we introduce the definition of certified unlearning~\citep{guo2020certified}.
\begin{definition}~\label{def:certified unlearning} ($\varepsilon-\delta$ certified unlearning)
Let $\mathcal{D}$ be a training set, $\mathcal{D}_u\subset\mathcal{D}$ be an unlearned set, $\mathcal{D}_r=\mathcal{D}\backslash\mathcal{D}_u$ be the retained set, $\mathcal{H}$ be the hypothesis space, and $\mathcal{A}$ be a learning process. $\mathcal{U}$ is an $\varepsilon-\delta$ certified unlearning process iff $\forall\;\mathcal{T}\subseteq\mathcal{H}$, we have
\begin{equation}
\scriptsize
\begin{aligned}
\mathrm{Pr}\left(\mathcal{U}\left(\mathcal{D},\mathcal{D}_u,\mathcal{A}\left(\mathcal{D}\right)\right)\in\mathcal{T}\right)\leq e^{\varepsilon}\mathrm{Pr}\left(\mathcal{U}\left(\mathcal{D}_r,\emptyset,\mathcal{A}\left(\mathcal{D}_r\right)\right)\in\mathcal{T}\right)+\delta, \\
\mathrm{Pr}\left(\mathcal{U}\left(\mathcal{D}_r,\emptyset,\mathcal{A}\left(\mathcal{D}_r\right)\right)\in\mathcal{T}\right)\leq e^{\varepsilon}\mathrm{Pr}\left(\mathcal{U}\left(\mathcal{D},\mathcal{D}_u,\mathcal{A}\left(\mathcal{D}\right)\right)\in\mathcal{T}\right)+\delta. \\
\end{aligned}
\end{equation}
\end{definition}

\noindent\textbf{\textit{Certified Unlearning and Differential Privacy.}}
It is worth noting that \cref{def:certified unlearning} is derived from the notion of differential privacy~\citep{dwork2006differential}. 
Specifically, a randomized algorithm $\mathcal{M}:\mathbb{N}^n\rightarrow\mathcal{R}$ gives $\varepsilon-\delta$ differential privacy iff $\forall\;\mathcal{T}\subseteq\mathcal{R}$ and $\forall\;\bm{x},\bm{y}\in\mathbb{N}^n$ such that $\|\bm{x}-\bm{y}\|_1=b$,
\begin{equation}
\mathrm{Pr}(\mathcal{M}(\bm{x})\in\mathcal{T})\leq e^{\varepsilon}\mathrm{Pr}(\mathcal{M}(\bm{y})\in\mathcal{T})+\delta.
\end{equation}
Here, $\bm{x}$ and $\bm{y}$ can be seen as two adjacent datasets with $b$ different data records.
Noting the similarity of certified unlearning and differential privacy, we clarify their relationship as follows.
\begin{proposition}\label{pro:relationship}
If learning algorithm $\mathcal{A}$ provides $\varepsilon-\delta$ differential privacy, $\mathcal{A}(\mathcal{D})$ is an $\varepsilon-\delta$ certified unlearned model.
\end{proposition}
\cref{pro:relationship} indicates that differential privacy is a sufficient condition of certified unlearning.
One of the advantages of differential privacy is its weak assumptions: differential privacy can be achieved without any requirements on convexity or continuity.
Such property perfectly fits DNNs, which are usually nonconvex.
Luckily, \cref{pro:relationship} enables achieving certified unlearning by borrowing the techniques in differential privacy.
Consequently, we have the following theorem for certified unlearning.
\begin{theorem}\label{thm:certified unlearning}
Let $\Tilde{\bm{w}}^*$ be the empirical minimizer over $\mathcal{D}_r$ and $\tilde{\bm{w}}=\mathcal{F}(\bm{w}^*,\mathcal{D}_u,\mathcal{D})$ be an approximation of $\Tilde{\bm{w}}^*$.
Define $\Delta$ as an upper bound of $\|\Tilde{\bm{w}}-\tilde{\bm{w}}^*\|_2$, then we have
$\mathcal{U}(\bm{w}^*,\mathcal{D}_u,\mathcal{D})=\tilde{\bm{w}}+Y$ is an $\varepsilon-\delta$ certified unlearning process, where $Y\sim\mathcal{N}(0,\sigma^2\bm{I})$ and $\sigma\geq\frac{\Delta}{\varepsilon}\sqrt{2\mathrm{ln(1.25/\delta)}}$.
\end{theorem}
The proof of \cref{thm:certified unlearning} is provided in \cref{sec:proof}.
With \cref{thm:certified unlearning}, we divide the problem of certified unlearning into two parts, \textit{estimating the retrained model} and \textit{adding noise corresponding to the approximation error}.
We illustrate the framework of certified unlearning in \cref{fig:diagram}.
Next, our primary goal is to find the approximation function $\mathcal{F}$ where the output $\tilde{\bm{w}}=\mathcal{F}(\bm{w}^*,\mathcal{D}_u,\mathcal{D})$ estimates the minimizer $\Tilde{\bm{w}}^*$ with a \textit{bounded approximation error}.

\noindent\textbf{\textit{Additional Remarks on Privacy and Unlearning.}}
The deduction of our certification budgets stems from results in differential privacy~\citep{dwork2014algorithmic}.
According to recent literature~\citep{balle2018improving} in differential privacy, we notice that our certification budgets can be further tightened, i.e., we can achieve $\varepsilon-\delta$ certified unlearning by adding a noise with a smaller variance $\sigma$.
In particular, we can find a smaller $\sigma$ by satisfying 
$\Phi(\frac{\Delta}{2\sigma}-\frac{\varepsilon\sigma}{\Delta})-e^\varepsilon\Phi(-\frac{\Delta}{2\sigma}-\frac{\varepsilon\sigma}{\Delta})\leq\delta$,
where $\Phi$ is the cumulative distribution function of the standard Gaussian. 
For simplicity, we still use the results in \cref{thm:certified unlearning} in the following discussion.

In addition, in the classical differential privacy~\citep{dwork2014algorithmic}, the value of privacy level $\varepsilon$ is less than $1$.
However, this condition is sometimes difficult to satisfy in certified unlearning for DNNs. 
Luckily, according to the Theorem 4 in~\citep{balle2018improving}, the certification of unlearning can be preserved when $\varepsilon\geq1$ by increasing the variance of the noise $\sigma$ (scaled by a factor $O(\sqrt{\varepsilon})$).

\begin{figure}[t]
\centering
\includegraphics[width=\linewidth]{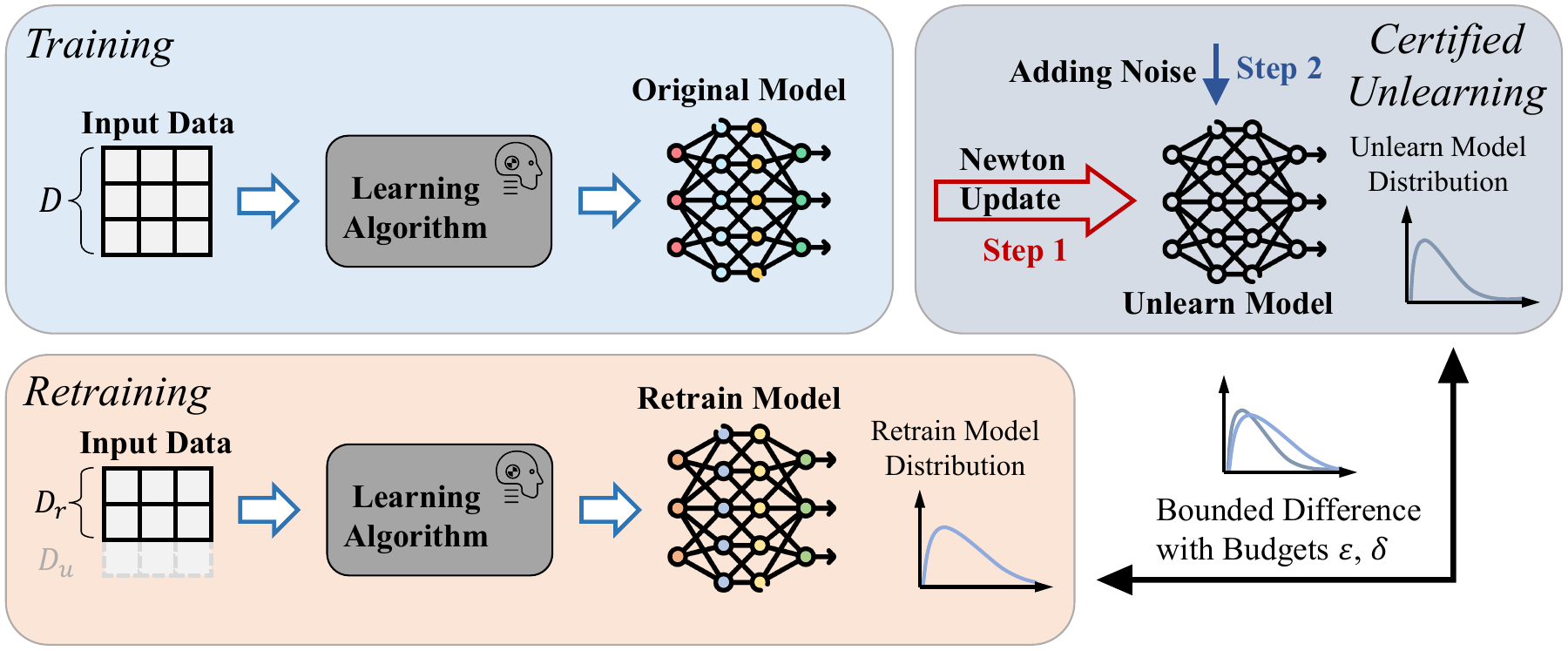}
\vspace{-5mm}
\caption{\footnotesize Illustration of certified unlearning, where the first step is to estimate the retrained model based on the original model, and the second step is to add noise to it. According to \cref{thm:certified unlearning}, we can guarantee the difference in distributions between the unlearned model and the retrained model is bounded by certification budgets.}
\label{fig:diagram}
\end{figure}

\section{Methodology}
\noindent\textbf{\textit{Existing Studies.}}
In existing studies~\citep{guo2020certified,golatkar2020eternal,golatkar2021mixed,mehta2022deep}, a single-step Newton update is widely used to estimate the empirical minimizer $\Tilde{\bm{w}}^*$ and then add noise to satisfy unlearning.
Assuming the second-order derivative of $\mathcal{L}$ with respect to $\bm{w}$ exists and is continuous, one can take the Taylor expansion of $\nabla\mathcal{L}$ at $\bm{w}^*$ as
\begin{equation}\label{eq:Taylor expansion}
\nabla\mathcal{L}(\Tilde{\bm{w}}^*,\mathcal{D}_r)\approx\nabla\mathcal{L}(\bm{w}^*,\mathcal{D}_r)+\bm{H}_{\bm{w}^*}(\Tilde{\bm{w}}^*-\bm{w}^*),
\end{equation}
where $\bm{H}_{\bm{w}^*}=\nabla^2\mathcal{L}(\bm{w}^*,\mathcal{D}_r)$ denotes the Hessian of $\mathcal{L}$ \textit{over} $\mathcal{D}_r$ at $\bm{w}^*$.
We also have $\nabla\mathcal{L}(\Tilde{\bm{w}}^*,\mathcal{D}_r)=\bm{0}$ since $\Tilde{\bm{w}}^*$ is the minimizer.
Move $\Tilde{\bm{w}}^*$ term to the left-hand side and we have
\begin{equation}\label{eq:influence function}
\Tilde{\bm{w}}^*\approx\tilde{\bm{w}}=\bm{w}^*-\bm{H}_{\bm{w}^*}^{-1}\nabla\mathcal{L}(\bm{w}^*,\mathcal{D}_r).
\end{equation}
In \cref{eq:influence function}, we consider the right-hand side as an approximation of $\Tilde{\bm{w}}^*$ and use $\tilde{\bm{w}}$ to denote the approximation value.
\cref{eq:influence function} cannot directly serve as the function $\mathcal{F}$ unless it satisfies two requirements simultaneously: \textbf{(1). bounded approximation error}; and \textbf{(2). computational efficiency}, which are extremely challenging due to the nature of DNNs.
However, we found that we can make the single-step Newton update satisfy these requirements by making simple revisions to \cref{eq:influence function}.
Next, we provide a detailed analysis.

\vspace{1.5mm}
\noindent\textbf{\textit{Approximation Error.}}
In \cref{eq:Taylor expansion}, we truncate the Taylor series and only consider the first-order term, hence \cref{eq:influence function} has an error in estimating the value of $\Tilde{\bm{w}}^*$.
Prior to further discussion, we first make the following assumptions.
\begin{assumption}\label{asp:continuity1}
The loss function $l(\bm{w},\bm{x})$ has an $L$-Lipschitz gradient in terms of $\bm{w}$.
\end{assumption}
\begin{assumption}\label{asp:continuity2}
The loss function $l(\bm{w},\bm{x})$ has an $M$-Lipschitz Hessian in terms of $\bm{w}$.
\end{assumption}
\Cref{asp:continuity1} and \Cref{asp:continuity2} are widely used in existing works~\citep{guo2020certified,sekhari2021remember,chien2023efficient,wu2023gif,wu2023certified}. 
In contrast to most existing works, our assumption does not necessitate the convexity of the objective, enabling our discussion to be applicable to DNNs in practice.
The following proposition provides a preliminary result on the approximation error bound.
\begin{lemma}\label{lmm:error bound}
Let $\bm{w}^*=\mathrm{argmin}_{\bm{w}\in\mathcal{H}}\mathcal{L}(\bm{w},\mathcal{D})$ and $\tilde{\bm{w}}^*=\mathrm{argmin}_{\bm{w}\in\mathcal{H}}\mathcal{L}(\bm{w},\mathcal{D}_r)$.
Let $\tilde{\bm{w}}=\bm{w}^*-\bm{H}_{\bm{w}^*}^{-1}\nabla\mathcal{L}(\bm{w}^*,\mathcal{D}_r)$ be an approximation of $\Tilde{\bm{w}}^*$. 
Consider \cref{asp:continuity2}, we have
\begin{equation}\label{eq:preliminary upper bound}
\|\tilde{\bm{w}}-\tilde{\bm{w}}^*\|_2\leq\frac{M}{2}\|\bm{H}_{\bm{w}^*}^{-1}\|_2\cdot\|\bm{w}^*-\tilde{\bm{w}}^*\|_2^2.
\end{equation}
\end{lemma}
The proof of \cref{lmm:error bound} can be found in \cref{sec:proof}. From \cref{lmm:error bound}, we can find that the approximation error bound is determined by two factors: $\|\bm{H}_{\bm{w}^*}^{-1}\|_2$ and $\|\bm{w}^*-\tilde{\bm{w}}^*\|_2^2$. 
Next, we aim to bound these two factors separately.

\noindent (1). The norm of the inverse Hessian $\|\bm{H}_{\bm{w}^*}^{-1}\|_2$: 
For nonconvex objective $\mathcal{L}(\bm{w},\mathcal{D}_r)$, $\|\bm{H}_{\bm{w}^*}^{-1}\|_2$ can be arbitrarily large.
To bound the value of $\|\bm{H}_{\bm{w}^*}^{-1}\|_2$, we exploit the \textbf{local convex approximation} technique~\citep{NoceWrig06}.
Specifically, we add an $\ell$-2 regularization term to the objective as $\mathcal{L}(\bm{w},\mathcal{D}_r)+\frac{\lambda}{2}\|\bm{w}\|_2^2$ {when computing the Hessian}.
Then, we have $\tilde{\bm{w}}=\bm{w}^*-(\bm{H}_{\bm{w}^*}+\lambda\bm{I})^{-1}\nabla\mathcal{L}(\bm{w}^*,\mathcal{D}_r)$ and the norm of the inverse Hessian changes to $\|(\bm{H}_{\bm{w}^*}+\lambda\bm{I})^{-1}\|_2$, correspondingly.
Intuitively, we use local convex approximation to make the nonconvex objective ($\lambda_{min}<0$) strongly convex with parameter $\lambda+\lambda_{min}$ at $\bm{w}^*$.

\noindent (2). The squared $\ell$-2 distance between $\bm{w}^*$ and $\tilde{\bm{w}}^*$: $\|\bm{w}^*-\tilde{\bm{w}}^*\|_2^2$: 
We find that $\|\bm{w}^*-\tilde{\bm{w}}^*\|_2^2$ is completely determined by the optimization problem in the learning process.
To bound the value of $\|\bm{w}^*-\tilde{\bm{w}}^*\|_2^2$, we modify the optimization problem by \textbf{adding a constraint $\|\bm{w}\|_2\leq C$ in the learning process}.
Then we have $\bm{w}^*=\mathrm{argmin}_{\|\bm{w}\|_2\leq C}\mathcal{L}(\bm{w},\mathcal{D})$ and $\tilde{\bm{w}}^*=\mathrm{argmin}_{\|\bm{w}\|_2\leq C}\mathcal{L}(\bm{w},\mathcal{D}_r)$, correspondingly.
It is worth noting that previous studies~\citep{guo2020certified,sekhari2021remember} assume the objective to be Lipschitz continuous and strongly convex simultaneously.
However, these assumptions are contradictory unless the norm of model parameters is bounded (see \cref{sec:assumption}).
In this paper, we directly achieve the important requirements of bounded model parameters which is implicitly required in previous works.
We resort to \textit{projected gradient descent}~\citep{bertsekas1997nonlinear} to solve this constrained optimization problem in the learning process.
By considering this constraint, we can obtain a worst-case upper bound of the term $\|\bm{w}^*-\tilde{\bm{w}}^*\|_2^2$ which is \textit{independent of the size of the unlearned set}.

By leveraging the local convex approximation and the constraint on the norm of the parameters, we can further derive a tractable approximation error bound.
\begin{theorem}\label{thm:upper bound}
Let $\bm{w}^*=\mathrm{argmin}_{\|\bm{w}\|_2\leq C}\mathcal{L}(\bm{w},\mathcal{D})$ and $\tilde{\bm{w}}^*=\mathrm{argmin}_{\|\bm{w}\|_2\leq C}\mathcal{L}(\bm{w},\mathcal{D}_r)$.
Denote $\lambda_{min}$ as the smallest eigenvalue of $\bm{H}_{\bm{w}^*}$.
Let $\tilde{\bm{w}}=\bm{w}^*-(\bm{H}_{\bm{w}^*}+\lambda\bm{I})^{-1}\nabla\mathcal{L}(\bm{w}^*,\mathcal{D}_r)$ be an approximation of $\Tilde{\bm{w}}^*$, where $\lambda>\|\bm{H}_{\bm{w}^*}\|_2$. 
Consider \cref{asp:continuity2}, we have
\begin{equation}\label{eq:upper bound}
\|\tilde{\bm{w}}-\tilde{\bm{w}}^*\|_2\leq\frac{2C(MC+\lambda)}{\lambda+\lambda_{min}}.
\end{equation}
\end{theorem}
The proof of \cref{thm:upper bound} is provided in \cref{sec:proof}.
\cref{thm:upper bound} provides a basic result for the approximation error bound.
In contrast to previous studies, our result does not require the objective to be convex.
We compare the approximation error bounds with and without the convexity assumption in \cref{sec:comparison} to show our progress.

\vspace{1.5mm}
\noindent\textbf{\textit{Computational Efficiency.}}
The main advantage of certified unlearning compared with retraining from scratch is computational efficiency.
Noting the importance of efficiency for certified unlearning, we aim to reduce the computational cost of $\tilde{\bm{w}}=\mathcal{F}(\bm{w}^*,\mathcal{D}_u,\mathcal{D})$ \textit{while keeping the approximation error bound}.
The computation cost contains two parts: the inverse Hessian $(\bm{H}_{\bm{w}^*}+\lambda\bm{I})^{-1}$ (with local convex approximation) and the gradient $\nabla\mathcal{L}(\bm{w}^*,\mathcal{D}_r)$.
First, assuming the number of learnable parameters is $p$, the computation of the inverse Hessian requires a complexity of $O(np^2+p^3)$.
Since the scale of the DNNs can be extremely large, we adopt a computationally efficient algorithm, LiSSA~\citep{agarwal2017second}, to estimate the inverse Hessian.
In particular, we let $X\in\mathcal{D}_r$ be a random variable of the retained training samples. 
Given $s$ independent and identically distributed (i.i.d.) retained training samples $\{X_1,\dots,X_s\}$, we can obtain $s$ i.i.d. samples $\{\bm{H}_{1,\lambda},\dots,\bm{H}_{s,\lambda}\}$ of the Hessian matrix $\bm{H}_{\bm{w}^*}+\lambda\bm{I}$, where $\bm{H}_{i,\lambda}=\nabla^2\mathcal{L}(\bm{w}^*,X_i)+\lambda\bm{I}$. 
Then, we can construct an estimator of the inverse Hessian. 
\begin{proposition}\label{pro:Hessian estimator}
Given $s$ i.i.d. samples $\{\bm{H}_{1,\lambda},\dots,\bm{H}_{s,\lambda}\}$ of the Hessian matrix $\bm{H}_{\bm{w}^*}+\lambda\bm{I}$, we let 
\begin{equation}\label{eq:Hessian estimator}
\small
\tilde{\bm{H}}^{-1}_{t,\lambda}=\bm{I}+\left(\bm{I}-\frac{\bm{H}_{t,\lambda}}{H}\right)\tilde{\bm{H}}^{-1}_{t-1,\lambda},\ \tilde{\bm{H}}^{-1}_{0,\lambda}=\bm{I},
\end{equation}
where $\|\nabla^2l(\bm{w}^*,x)+\lambda\bm{I}\|\leq H$, $\forall x\in\mathcal{D}_r$.
Then, we have that $\frac{\tilde{\bm{H}}^{-1}_{s,\lambda}}{H}$ is an asymptotic unbiased estimator of the inverse Hessian $(\bm{H}_{\bm{w}^*}+\lambda\bm{I})^{-1}$.
\end{proposition}
The proof of \cref{pro:Hessian estimator} can be found in \cref{sec:proof}.
By incorporating the Hessian-vector product technique, we can compute $\tilde{\bm{H}}^{-1}_{s,\lambda}\nabla\mathcal{L}(\bm{w}^*,\bm{D}_r)$ with a complexity of $O(sp^2)$ which is independent of $n$.
Second, to compute the gradient $\nabla\mathcal{L}(\bm{w}^*,\mathcal{D}_r)$ efficiently, we leverage the property of the minimizer $\bm{w}^*$ that $\nabla\mathcal{L}(\bm{w}^*,\mathcal{D})=\frac{n_u}{n}\nabla\mathcal{L}(\bm{w}^*,\mathcal{D}_u)+\frac{n-n_u}{n}\nabla\mathcal{L}(\bm{w}^*,\mathcal{D}_r)=0$. 
Noting that the number of unlearned samples is usually much smaller than the number of retained samples, we can replace the gradient $\nabla\mathcal{L}(\bm{w}^*,\mathcal{D}_r)$ as $\nabla\mathcal{L}(\bm{w}^*,\mathcal{D}_r)=-\frac{n_u}{n-n_u}\nabla\mathcal{L}(\bm{w}^*,\mathcal{D}_u)$ and reduce the complexity from $O(np)$ to $O(n_up)$ where $n_u\ll n$. 
After exploiting the two efficient techniques, we prove the approximation error is still bounded as follows.
\begin{theorem}\label{thm:efficient upper bound}
Let $\bm{w}^*=\mathrm{argmin}_{\|\bm{w}\|_2\leq C}\mathcal{L}(\bm{w},\mathcal{D})$ and $\tilde{\bm{w}}^*=\mathrm{argmin}_{\|\bm{w}\|_2\leq C}\mathcal{L}(\bm{w},\mathcal{D}_r)$.
Let $\lambda_{min}$ be the smallest eigenvalue of $\bm{H}_{\bm{w}^*}$, and $s$ be the recursion number for the inverse Hessian approximation.
Let $\tilde{\bm{w}}=\bm{w}^*+\frac{n_u}{(n-n_u)H}\tilde{\bm{H}}^{-1}_{s,\lambda}\nabla\mathcal{L}(\bm{w}^*,\mathcal{D}_u)$ be an approximation of $\Tilde{\bm{w}}^*$, where $\lambda>\|\bm{H}_{\bm{w}^*}\|_2$. 
Consider \cref{asp:continuity1} and \cref{asp:continuity2}, when $s\geq2\frac{L+\lambda}{\lambda+\lambda_{\min}}\mathrm{ln}\frac{L+\lambda}{\lambda+\lambda_{\min}}$, the following inequality holds with a probability larger than $1-\rho$. 
\footnote{For \Cref{thm:efficient upper bound} and \Cref{pro:practical upper bound}, a correction has been made to the upper bounds, and a condition on $s$ is added to this version. 
The original bound incorrectly adopted the result from \citep{agarwal2017second} due to a misunderstanding of the notation. 
The corrected bound includes an additional multiplicative factor of $(\lambda+L)$ in the second term. This does not affect the technical contributions or conclusions of this paper.}
\begin{equation}\label{eq:efficient upper bound}
\scriptsize
\|\tilde{\bm{w}}-\tilde{\bm{w}}^*\|_2\leq\frac{2C(MC+\lambda)}{\lambda+\lambda_{min}}+\left(\frac{32\sqrt{\mathrm{ln}\,d/\rho}(\lambda+L)}{\lambda+\lambda_{min}}+\frac{1}{8}\right)LC.
\end{equation}
\end{theorem}
The proof of \cref{thm:efficient upper bound} is provided in \cref{sec:proof}.
Finally, we obtain a computationally efficient $\mathcal{F}$ with bounded approximation error as
\begin{equation}\label{eq:final approximation}
\small
\mathcal{F}(\bm{w}^*,\mathcal{D}_u,\mathcal{D})=\bm{w}^*+\frac{n_u}{(n-n_u)H}\tilde{\bm{H}}^{-1}_{s,\lambda}\nabla\mathcal{L}(\bm{w}^*,\mathcal{D}_u).
\end{equation}
Based on our discussion, the complexity of computing $\mathcal{F}$ is $O(sp^2+n_up)$.
Consequently, we can achieve certified unlearning based on $\mathcal{F}$ according to \cref{thm:certified unlearning}.
We formulate the overall algorithm for certified unlearning in \cref{alg:certified unlearning}.


\begin{algorithm}[t]
\caption{Single-Batch Certified Unlearning for DNNs}\label{alg:certified unlearning}
\textbf{Input}: original trained model $\bm{w}^*$; certification budget $\varepsilon$ and $\delta$; local convex coefficient $\lambda$; norm upper bound $C$; Hessian norm bound $H$; number of recursion $s$; unlearned set $\mathcal{D}_u$. \\
\textbf{Output}: certified unlearning model $\bm{w}^-$. \\
\vspace{-4mm}
\begin{algorithmic}[1] 
\STATE Compute $s$ i.i.d. Hessian samples $\{\bm{H}_{1,\lambda},\dots,\bm{H}_{s,\lambda}\}$.
\STATE $\bm{P}_{0,\lambda}\leftarrow\nabla\mathcal{L}(\bm{w}^*,\mathcal{D}_u)$.
\FOR{$j=1,\dots,s$}
\STATE $\bm{P}_{j,\lambda}\leftarrow\nabla\mathcal{L}(\bm{w}^*,\mathcal{D}_u)+(\bm{I}-\frac{\bm{H}_{j,\lambda}}{H})\bm{P}_{j-1,\lambda}$.
\ENDFOR
\STATE $\tilde{\bm{w}}\leftarrow\bm{w}^*+\frac{n_u}{(n-n_u)H}\bm{P}_{s,\lambda}$.
\STATE Compute the error bound $\Delta$ based on \cref{eq:practical upper bound}.
\STATE $\sigma\leftarrow\frac{\Delta}{\varepsilon}\sqrt{2\mathrm{ln(1.25/\delta)}}$.
\STATE $\bm{w}^-\leftarrow\tilde{\bm{w}}+Y$, where $Y\sim\mathcal{N}(0,\sigma^2\bm{I})$.
\end{algorithmic}
\end{algorithm}

\section{Practical Consideration}
\noindent\textbf{\textit{Nonconvergence.}}
In previous results, we let $\bm{w}^*$ and $\tilde{\bm{w}}^*$ be the exact minimum of $\mathcal{L}$ over $\mathcal{D}$ and $\mathcal{D}_r$, respectively.
However, the training process of DNNs usually stops before reaching the exact minimum, e.g., early stopping, for better generalization.
Hereby, we analyze the approximation error bound under the nonconvergence condition.
In the nonconvergence case, $\bm{w}^*$ and $\tilde{\bm{w}}^*$ are not the exact minimum which is intractable.
We further assume that the norm of the gradient $\|\nabla\mathcal{L}(\bm{w}^*,\mathcal{D})\|$ and $\|\nabla\mathcal{L}(\tilde{\bm{w}}^*,\mathcal{D}_r)\|$ are bounded by $G$.
Consequently, we have the adjusted approximation error bound in nonconvergence cases.
\begin{proposition}\label{pro:practical upper bound}
Let $\bm{w}^*$ and $\tilde{\bm{w}}^*$ be two learned model in practice such that $\|\nabla\mathcal{L}(\bm{w}^*,\mathcal{D})\|,\|\nabla\mathcal{L}(\tilde{\bm{w}}^*,\mathcal{D}_r)\|\leq G$ and $\|\bm{w}^*\|,\|\tilde{\bm{w}}^*\|\leq C$.
Let $\lambda_{min}$ be the smallest eigenvalue of $\bm{H}_{\bm{w}^*}$, and $s$ be the recursion number for the inverse Hessian approximation.
Let $\tilde{\bm{w}}=\bm{w}^*+\frac{n_u}{(n-n_u)H}\tilde{\bm{H}}^{-1}_{s,\lambda}\nabla\mathcal{L}(\bm{w}^*,\mathcal{D}_u)$ be an approximation of $\Tilde{\bm{w}}^*$, where $\lambda>\|\bm{H}_{\bm{w}^*}\|_2$. 
Consider \cref{asp:continuity1} and \cref{asp:continuity2}, when $s\geq2\frac{L+\lambda}{\lambda+\lambda_{\min}}\mathrm{ln}\frac{L+\lambda}{\lambda+\lambda_{\min}}$, the following inequality holds with a probability larger than $1-\rho$.
\begin{equation}\label{eq:practical upper bound}
\scriptsize
\begin{aligned}
\|\tilde{\bm{w}}-\tilde{\bm{w}}^*\|_2&\leq\frac{2C(MC+\lambda)+G}{\lambda+\lambda_{min}} \\
&\quad\quad+\left(\frac{16\sqrt{\mathrm{ln}\,d/\rho}(\lambda+L)}{\lambda+\lambda_{min}}+\frac{1}{16}\right)(2LC+G).
\end{aligned}
\end{equation}
\end{proposition}
The proof of \cref{pro:practical upper bound} is provided in \cref{sec:proof}. 
With \cref{pro:practical upper bound}, we can derive a more precise approximation error bound in practice, knowing the tractable residual gradient of the learning algorithm $\mathcal{A}$: $\|\nabla\mathcal{L}(\bm{w}^*,\mathcal{D})\|$.

\begin{algorithm}[t]
\caption{Sequential Unlearning for DNNs}\label{alg:sequential unlearning}
\textbf{Input}: original trained model $\bm{w}^*$; certification budget $\varepsilon$ and $\delta$; local convex coefficient $\lambda$; norm upper bound $C$; number of recursion $s$; unlearned sets $\{\mathcal{D}_{u_1},\dots,\mathcal{D}_{u_k}\}$. \\
\textbf{Output}: certified unlearning model $\bm{w}^-$. \\
\vspace{-4mm}
\begin{algorithmic}[1] 
\STATE $\tilde{\bm{w}}_0\leftarrow\bm{w}^*$, $\mathcal{D}_{r_0}\leftarrow\mathcal{D}$.
\FOR{$i=1,\dots,k$}
\STATE $\mathcal{D}_{r_i}\leftarrow\mathcal{D}_{r_{i-1}}\backslash\mathcal{D}_{u_i}$.
\STATE Obtain $\nabla\mathcal{L}(\tilde{\bm{w}}_{k-1},\mathcal{D}_{r_k})$ and $\{\bm{H}_{1,\lambda,i},\dots,\bm{H}_{s,\lambda,i}\}$.
\STATE $\bm{P}_{0,\lambda,i}\leftarrow\nabla\mathcal{L}(\tilde{\bm{w}}_{k-1},\mathcal{D}_{r_k})$.
\FOR{$j=1,\dots,s$}
\STATE $\bm{P}_{j,\lambda,i}\leftarrow\nabla\mathcal{L}(\tilde{\bm{w}}_{k-1},\mathcal{D}_{r_k})+(\bm{I}-\frac{\bm{H}_{j,\lambda,i}}{H})\bm{P}_{j-1,\lambda,i}$.
\ENDFOR
\STATE $\tilde{\bm{w}}_i\leftarrow\tilde{\bm{w}}_{i-1}-\frac{\bm{P}_{s,\lambda,i}}{H}$.
\ENDFOR
\STATE Compute the error bound $\Delta$ based on \cref{eq:practical upper bound}.
\STATE $\sigma\leftarrow\frac{\Delta}{\varepsilon}\sqrt{2\mathrm{ln(1.25/\delta)}}$.
\STATE $\bm{w}^-\leftarrow\tilde{\bm{w}}_k+Y$, where $Y\sim\mathcal{N}(0,\sigma^2\bm{I})$.
\end{algorithmic}
\end{algorithm}

\noindent\textbf{\textit{Sequential Unlearning.}}
In practical scenarios, the user can send unlearning requests at different time points in a sequential order~\citep{guo2020certified,nguyen2022survey}.
For example, after unlearning one user's data, another user sends the unlearning request.
In such a case, we should unlearn the second user's data based on the unlearned model of the first user's data.
Consequently, certified unlearning should be able to work in a sequential setting to fit real-world scenarios.
We next show that our certified unlearning algorithm can be easily implemented in a sequential manner.
Let $\mathcal{D}_{u_k}$ be the unlearned set of the $k$-th unlearning request and $\mathcal{D}_{r_k}=\mathcal{D}\backslash\cup_{i=1}^k\mathcal{D}_{u_i}$ be the retained set after the $k$-th unlearning request. 
We estimate the retrained model sequentially as $\tilde{\bm{w}}_{k}=\tilde{\bm{w}}_{k-1}-\frac{1}{H}\tilde{\bm{H}}^{-1}_{s,\lambda,k-1}\nabla\mathcal{L}(\tilde{\bm{w}}_{k-1},\mathcal{D}_{r_k})$ where $\tilde{\bm{H}}^{-1}_{s,\lambda,k-1}$ is the approximation of the inverse Hessian over $\mathcal{D}_{r_k}$ at $\tilde{\bm{w}}_{k-1}$. 
Note that in sequential approximation, we cannot use the efficient computation of $\nabla\mathcal{L}(\tilde{\bm{w}}_{k-1},\mathcal{D}_{r_k})$ as $\tilde{\bm{w}}_{k-1}$ is not a minimum.
After estimating the retrained model after the $k$-th unlearning request, we add noise to obtain the unlearned model according to \cref{thm:certified unlearning}.
We can prove that our sequential approximation also has the same approximation error bound as single-step cases.
\begin{proposition}\label{pro:sequential upper bound}
Let $\tilde{\bm{w}}_{0}$ be a model trained on $\mathcal{D}$, $\bm{w}^*_{k}$ be a model trained on $\mathcal{D}_{r_k}$, and
$\tilde{\bm{w}}_{k}=\tilde{\bm{w}}_{k-1}-\tilde{\bm{H}}^{-1}_{s,\lambda,k-1}\nabla\mathcal{L}(\tilde{\bm{w}}_{k-1},\mathcal{D}_{r_k})$ be an approximation of $\bm{w}_k^*=\mathrm{argmin}_{\|\bm{w}\|_2\leq C}\mathcal{L}(\bm{w},\mathcal{D}_{r_k})$ for $k=1,2,\dots$, where $\lambda_{min}$ is the smallest eigenvalue of $\nabla^2\mathcal{L}(\tilde{\bm{w}}_{k-1},\mathcal{D}_{r_k})$ and $\lambda>\|\bm{H}_{\bm{w}^*}\|_2$. 
Consider \cref{asp:continuity1} and \cref{asp:continuity2} and assume $\|\nabla\mathcal{L}(\tilde{\bm{w}}_k,\mathcal{D}_{r_{k-1}})\|,\|\nabla\mathcal{L}(\bm{w}^*_k,\mathcal{D}_{r_k})\|\leq G$ for $k=1,2,\dots$, the sequential approximation error $\|\tilde{\bm{w}}_{k}-\bm{w}_k^*\|_2$ has the same upper bound as \cref{pro:practical upper bound}.
\end{proposition}
The proof of \cref{pro:sequential upper bound} is provided in \cref{sec:proof}.
\cref{pro:sequential upper bound} provides a theoretical guarantee on the certification for sequential unlearning.
As mentioned before, our constraint on model parameters provides a worst-case upper bound of $\|\bm{w}^*-\tilde{\bm{w}}^*\|_2^2$ independent on $\mathcal{D}_u$ so that our error bound in \cref{eq:practical upper bound} remains unchanged as the number of unlearned data gradually increases.
We illustrate the details of sequential unlearning in \cref{alg:sequential unlearning}.
It is worth noting that we fix the number of unlearned samples in a single unlearning process to be $b$ in Definition~\ref{def:certified unlearning}.
During the sequential unlearning process, the number of unlearned samples keeps increasing.
We can adapt the value of $b$ after each unlearning request to fit the size of the current unlearned set in a post hoc manner.
Another rigorous way to preserve the certification while fixing the unlearning granularity $b$ is to incorporate the group privacy notion~\citep{dwork2014algorithmic} into the sequential unlearning process.
Specifically, we can fix the unlearning granularity as the number of unlearning samples in a single unlearning request by linearly increasing the certification budget $\varepsilon$ as $k\varepsilon$, where $k$ is the number of the unlearning requests.

\vspace{1.5mm}
\noindent\textbf{\textit{Reducing the Approximation Error Bound.}}
With a lower approximation error bound, we can achieve the same certification budget by adding a \textit{smaller} noise, leading to a smaller impact on the performance of the unlearned model.
Therefore, we summarize our results and list several ways to reduce the approximation error bound in practice.
\begin{itemize}[leftmargin=*]
\item \textbf{Increasing the value of $\lambda$.} According to \cref{eq:practical upper bound}, we can easily obtain that the approximation error bound decreases after increasing the value of $\lambda$. 
\item \textbf{Decreasing the value of $C$.} Similarly, we can reduce the approximation error bound by decreasing the value of $C$. 
However, reducing $C$ can impact the effectiveness of the original model. 
We should carefully select a proper $C$ without compromising the model performance distinctly.
\item \textbf{Increasing the regularization.} Adding a larger regularization would increase the value of $\lambda_{min}$.
According to \cref{eq:practical upper bound}, the approximation error bound decreases when $\lambda_{min}$ increases, which is consistent with the empirical results in~\citep{basu2021influence}.
\end{itemize}

\begin{table*}[t]
\small
\renewcommand{\arraystretch}{1.05}
\tabcolsep = 3pt
\centering
\caption{Comparison between the certified unlearning method and unlearning baselines over three popular DNNs across three real-world datasets. We record the micro F1-score of the predictions on the unlearned set $\mathcal{D}_u$, retained set $\mathcal{D}_r$, and test set $\mathcal{D}_t$.}
\label{tab:performance}
\vspace{1mm}
\aboverulesep = 0pt
\belowrulesep = 0pt
\begin{tabular}{c|ccc|ccc|ccc}
\toprule
\multirow{2}{*}{\textbf{Method}} & \multicolumn{3}{c|}{\textbf{MLP \& MNIST}} & \multicolumn{3}{c|}{\textbf{AllCNN \& CIFAR-10}} & \multicolumn{3}{c}{\textbf{ResNet18 \& SVHN}} \\
& F1 on $\mathcal{D}_u$ & F1 on $\mathcal{D}_r$ & F1 on $\mathcal{D}_t$ & F1 on $\mathcal{D}_u$ & F1 on $\mathcal{D}_r$ & F1 on $\mathcal{D}_t$ & F1 on $\mathcal{D}_u$ & F1 on $\mathcal{D}_r$ & F1 on $\mathcal{D}_t$ \\
\midrule
Original & 98.30 {\scriptsize$\pm$ 0.51} & 98.37 {\scriptsize$\pm$ 0.06} & 97.50 {\scriptsize$\pm$ 0.08} & 87.97 {\scriptsize$\pm$ 3.01} & 90.71 {\scriptsize$\pm$ 1.11} & 83.04 {\scriptsize$\pm$ 0.20} & 94.53 {\scriptsize$\pm$ 0.74} & 95.00 {\scriptsize$\pm$ 0.47} & 93.26 {\scriptsize$\pm$ 0.34} \\
Retrain & 97.20 {\scriptsize$\pm$ 0.29} & 98.27 {\scriptsize$\pm$ 0.09} & 97.19 {\scriptsize$\pm$ 0.15} & 82.67 {\scriptsize$\pm$ 0.57} & 90.10 {\scriptsize$\pm$ 0.98} & 82.39 {\scriptsize$\pm$ 0.98} & 93.13 {\scriptsize$\pm$ 0.80} & 95.73 {\scriptsize$\pm$ 0.60} & 93.34 {\scriptsize$\pm$ 0.15} \\
\hline
Fine Tune & 97.67 {\scriptsize$\pm$ 0.33} & 98.35 {\scriptsize$\pm$ 0.15} & 97.22 {\scriptsize$\pm$ 0.09} & 89.84 {\scriptsize$\pm$ 1.99} & 92.25 {\scriptsize$\pm$ 0.26} & 84.31 {\scriptsize$\pm$ 0.54} & 93.73 {\scriptsize$\pm$ 0.33} & 95.23 {\scriptsize$\pm$ 0.47} & 93.75 {\scriptsize$\pm$ 0.41} \\
Neg Grad & 97.83 {\scriptsize$\pm$ 0.59} & 98.09 {\scriptsize$\pm$ 0.28} & 97.22 {\scriptsize$\pm$ 0.16} & 79.60 {\scriptsize$\pm$ 1.92} & 85.98 {\scriptsize$\pm$ 2.97} & 78.66 {\scriptsize$\pm$ 1.94} & 92.13 {\scriptsize$\pm$ 0.77} & 92.10 {\scriptsize$\pm$ 0.60} & 92.10 {\scriptsize$\pm$ 0.60} \\
Fisher & 97.70 {\scriptsize$\pm$ 0.90} & 97.56 {\scriptsize$\pm$ 0.14} & 96.69 {\scriptsize$\pm$ 0.05} & 87.97 {\scriptsize$\pm$ 3.69} & 90.71 {\scriptsize$\pm$ 1.36} & 83.04 {\scriptsize$\pm$ 0.24} & 94.23 {\scriptsize$\pm$ 0.91} & 94.84 {\scriptsize$\pm$ 0.49} & 93.06 {\scriptsize$\pm$ 0.29} \\
L-CODEC & 98.27 {\scriptsize$\pm$ 0.61} & 98.35 {\scriptsize$\pm$ 0.07} & 97.46 {\scriptsize$\pm$ 0.09} & 88.20 {\scriptsize$\pm$ 3.70} & 90.98 {\scriptsize$\pm$ 1.28} & 83.33 {\scriptsize$\pm$ 0.23} & 95.00 {\scriptsize$\pm$ 1.06} & 95.83 {\scriptsize$\pm$ 0.35} & 93.53 {\scriptsize$\pm$ 0.08} \\
\textbf{Certified} & 97.60 {\scriptsize$\pm$ 0.96} & 98.28 {\scriptsize$\pm$ 0.05} & 97.37 {\scriptsize$\pm$ 0.11} & 87.83 {\scriptsize$\pm$ 3.62} & 90.68 {\scriptsize$\pm$ 1.32} & 83.04 {\scriptsize$\pm$ 0.38} & 93.73 {\scriptsize$\pm$ 0.76} & 94.61 {\scriptsize$\pm$ 0.57} & 92.94 {\scriptsize$\pm$ 0.49} \\
\bottomrule
\end{tabular}
\vspace{-3mm}
\end{table*}

\section{Experiments}
\subsection{Dataset Information}
We conduct experiments based on three widely adopted real-world datasets for image classification, MNIST~\citep{lecun1998gradient}, SVHN~\citep{netzer2011reading}, and CIFAR-10~\citep{krizhevsky2009learning} to evaluate certified unlearning for DNNs.
The MNIST dataset consists of a collection of 60,000 handwritten digit images for training and 10,000 images for testing. 
The CIFAR-10 dataset contains 60,000 color images in 10 classes, with each class representing a specific object category. The dataset is split into 50,000 training images and 10,000 test images.
The SVHN dataset consists of house numbers images captured from Google Street View. 
The dataset is divided into three subsets: 73,257 training images, 26,032 test images, and 531,131 extra images. 
We only use the training set and the test set.

\subsection{Baseline Information}
To evaluate unlearning methods, we choose three different types of original models for the image classification task.
In particular, we train a three-layer MLP model, an AllCNN~\citep{springenberg2015striving} model, and a ResNet18~\citep{he2016deep} model for MNIST, CIFAR-10, and SVHN, respectively. 
Moreover, we compare the performance of our certified unlearning method with five unlearning baselines shown as follows. 
\textbf{Retrain from scratch}: retraining from scratch is an ideal unlearning baseline since it is usually considered as an exact unlearning method;
\textbf{Fine tune}~\citep{golatkar2020eternal}: fine-tuning the original model on the retained set $\mathcal{D}_r$ for one epoch after training the original model;
\textbf{Negative gradient}~\citep{golatkar2020eternal}: conducting the gradient ascent based on the gradient in terms of the unlearned set $\mathcal{D}_u$ for one epoch after training the original model;
\textbf{Fisher forgetting}~\citep{golatkar2020eternal}: after training the original model, the Fisher forgetting baseline exploits the Fisher information matrix to substitute the Hessian in the single-step Newton update;
\textbf{L-CODEC}~\citep{mehta2022deep}: after training the original model, L-CODEC selects a subset of the model parameters to compute the Hessian based on the Fisher information matrix for efficiency.

\subsection{Implementation}
We implemented all experiments in the PyTorch~\citep{paszke2019pytorch} package and exploited Adam~\citep{kingma2015adam} as the optimizer.
All experiments are implemented on an Nvidia RTX A6000 GPU.
We reported the average value and the standard deviation of the numerical results under three random seeds. 
For the values of $L$ and $M$ in our theoretical results, since finding a practical upper bound of the Lipschitz constant can be intractable for real-world tasks, we follow most previous works~\citep{koh2017understanding,wu2023gif,wu2023certified} to set them as hyperparameters which can be adjusted flexibly to adapt to different scenarios. 
The choice of these hyperparameters will not affect the soundness of our theoretical results, but render an imprecise value of the certification level, discussed in \Cref{sec:implementation}.
More detailed hyperparameter settings of certified unlearning and baselines can be found in \cref{sec:implementation} and \cref{sec:parameter}.
Our code is available at \url{https://github.com/zhangbinchi/certified-deep-unlearning}.

\begin{table*}[t]
\small
\renewcommand{\arraystretch}{1.05}
\tabcolsep = 3.6pt
\centering
\caption{Comparison between the certified unlearning method and unlearning baselines over three popular DNNs across three real-world datasets. We record the relearn time, the accuracy of the membership inference attack, and the AUC score of the membership inference attack for measuring the unlearning performance.}
\label{tab:unlearn metric}
\vspace{1mm}
\aboverulesep = 0pt
\belowrulesep = 0pt
\begin{tabular}{c|ccc|ccc|ccc}
\toprule
\multirow{2}{*}{\textbf{Method}} & \multicolumn{3}{c|}{\textbf{MLP \& MNIST}} & \multicolumn{3}{c|}{\textbf{AllCNN \& CIFAR-10}} & \multicolumn{3}{c}{\textbf{ResNet18 \& SVHN}} \\
& Relearn T & Attack Acc & Attack AUC & Relearn T & Attack Acc & Attack AUC & Relearn T & Attack Acc & Attack AUC \\
\midrule
Retrain & 25 & 93.10 {\scriptsize$\pm$ 0.33} & 95.16 {\scriptsize$\pm$ 0.47} & 17 & 79.82 {\scriptsize$\pm$ 0.35} & 88.71 {\scriptsize$\pm$ 0.43} & 7 & 90.47 {\scriptsize$\pm$ 0.14} & 93.07 {\scriptsize$\pm$ 0.27} \\
\hline
Fine Tune & 17 & 93.65 {\scriptsize$\pm$ 0.23} & 95.37 {\scriptsize$\pm$ 0.46} & 14 & 79.42 {\scriptsize$\pm$ 1.05} & 88.13 {\scriptsize$\pm$ 0.66} & 7 & 90.63 {\scriptsize$\pm$ 0.32} & 92.96 {\scriptsize$\pm$ 0.31} \\
Neg Grad & 21 & 93.73 {\scriptsize$\pm$ 0.45} & 95.42 {\scriptsize$\pm$ 0.43} & 17 & 78.63 {\scriptsize$\pm$ 1.23} & 87.58 {\scriptsize$\pm$ 0.96} & 9 & 90.02 {\scriptsize$\pm$ 0.13} & 92.89 {\scriptsize$\pm$ 0.22} \\
Fisher & 21 & 93.85 {\scriptsize$\pm$ 0.22} & 95.37 {\scriptsize$\pm$ 0.51} & 14 & 79.70 {\scriptsize$\pm$ 1.03} & 88.58 {\scriptsize$\pm$ 0.76} & 9 & 90.47 {\scriptsize$\pm$ 0.84} & 93.13 {\scriptsize$\pm$ 0.19} \\
L-CODEC & 20 & 95.05 {\scriptsize$\pm$ 0.05} & 95.31 {\scriptsize$\pm$ 0.21} & 14 & 83.60 {\scriptsize$\pm$ 0.62} & 92.18 {\scriptsize$\pm$ 0.17} & 7 & 93.22 {\scriptsize$\pm$ 0.35} & 93.75 {\scriptsize$\pm$ 0.54} \\
\textbf{Certified} & 24 & 93.22 {\scriptsize$\pm$ 0.46} & 95.28 {\scriptsize$\pm$ 0.50} & 25 & 78.00 {\scriptsize$\pm$ 1.18} & 87.22 {\scriptsize$\pm$ 1.13} & 9 & 88.63 {\scriptsize$\pm$ 1.58} & 92.18 {\scriptsize$\pm$ 1.16} \\
\bottomrule
\end{tabular}
\vspace{-3mm}
\end{table*}

\subsection{Unlearning Performance}
To evaluate the performance of all unlearning baselines, we select different unlearning metrics according to existing studies.
\cref{tab:performance} exhibits the model utility, i.e., micro F1-score of the predictions over the unlearned set $\mathcal{D}_u$, retained set $\mathcal{D}_r$, and the test set $\mathcal{D}_t$.
Based on the meaning of machine unlearning, we expect the results of a desirable unlearning method to be \textit{close to the results of the retrained model}.
Hence, we use the retrain-from-scratch baseline as a standard for evaluating the unlearning baselines.
Regarding the performance of unlearning baselines, we have the following observations derived from the comparison of our method with unlearning baselines.
(1). The micro F1-score of the certified unlearning method on the unlearned set $\mathcal{D}_u$ is closest to the retrained model for most cases.
(2). The micro F1-score of the certified unlearning method on the retained set $\mathcal{D}_r$ and the test set $\mathcal{D}_t$ is closest to the retrained model on MNIST and CIFAR-10.
(3). Compared with the original training, certified unlearning obtains an unlearned model with a lower F1-score over $\mathcal{D}_u$, $\mathcal{D}_r$, and $\mathcal{D}_t$, but the utility over $\mathcal{D}_u$ always decreases the most, consistent with the goal of unlearning.
In general, we can find that the negative gradient baseline leads to a large utility drop over all subsets and the L-CODEC baseline yields a distinct utility increase over all subsets.
In contrast, our method has a lower prediction utility over the unlearned set $\mathcal{D}_u$ and maintains a desirable utility over the retained set $\mathcal{D}_r$ and the test set $\mathcal{D}_t$.

\begin{figure}[t]
    \centering
    \vspace{2mm}
    \includegraphics[width=0.9\linewidth]{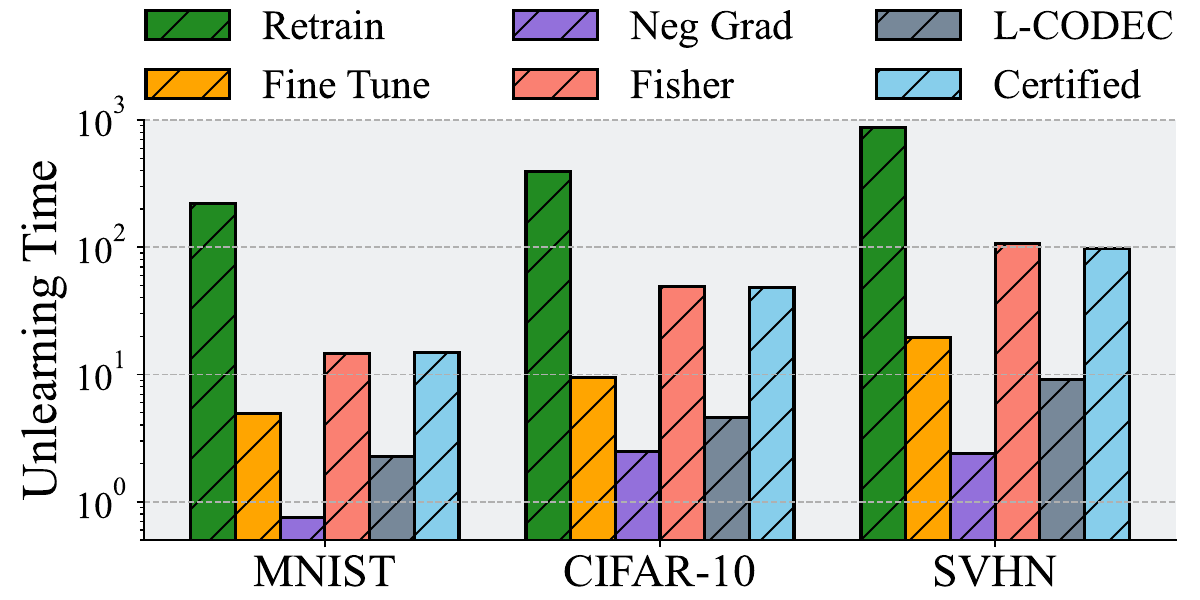}
    \vspace{-3mm}
    \caption{Comparison of unlearning time between the certified unlearning method and unlearning baselines over three popular DNNs across three datasets.}
    \label{fig:efficiency}
    \vspace{-5mm}
\end{figure}

\begin{table}[t]
\renewcommand{\arraystretch}{1.05}
\tabcolsep = 8pt
\centering
\caption{The value of approximation error bound, approximation error, and approximation error without constraint under different values of the local convex coefficient $\lambda$ over the AllCNN backbone on CIFAR-10.}
\label{tab:ablation}
\vspace{1mm}
\aboverulesep = 0pt
\belowrulesep = 0pt
\begin{tabular}{cccc}
\toprule
$\lambda$ & Err Bound & Approx Err & Approx Err (N) \\
\midrule
0 & 35077.7545 & 26.5146 & 46.8936 \\
10 & 390.3298 & 26.1841 & 46.7587 \\
$10^2$ & 78.4548 & 26.1724 & 46.7240 \\
$10^3$ & 46.9586 & 26.1692 & 46.7091 \\
$10^4$ & 43.8059 & 26.1688 & 46.7035 \\
\bottomrule
\end{tabular}
\vspace{-8mm}
\end{table}

In \cref{tab:unlearn metric}, we record several different unlearning metrics, i.e., relearn time, F1-score of membership inference attack, and AUC score of membership inference attack.
By relearning the unlearned model over the unlearned set, relearn time is the number of epochs for the loss function over the unlearned set $\mathcal{D}_u$ to descend to a fixed threshold.
We tried using the loss value of the original model over the forget set as the threshold but led to a very small relearn time that cannot discriminate any two baselines (most methods can recover the performance within 1 epoch). 
Finally, we artificially choose the threshold small enough to make the results of different baselines distinguishable.
As an unlearning metric, a lower relearn time indicates the unlearned model retains more information of the unlearned samples.
Moreover, we use the membership inference attack to evaluate the information leakage of the unlearning methods.
Membership inference attack trains an adversary to infer whether a data sample is used to train the unlearned model.
We choose the state-of-the-art membership inference attack method~\citep{chen2021machine} for evaluating machine unlearning methods.
Specifically, we use the accuracy and the AUC score to measure the utility of the membership inference attack.
A higher attack success rate indicates the unlearned model contains more information of the unlearned samples.
The experimental results demonstrate that our method is more effective than other unlearning baselines in removing the information of the unlearned samples.
It is worth noting that the retraining baseline is less desirable than other unlearning methods in some cases regarding the membership inference attack results.
The reason is that the retraining baseline is directly used to train the shadow unlearned model in the membership inference attack framework~\citep{chen2021machine} while the attack model is transferred to attack other unlearning baselines.
In conclusion, our method can effectively remove the information of unlearned samples from the unlearned model and preserve the privacy of unlearned samples.

\subsection{Efficiency}
Efficiency is a crucial advantage of certified unlearning compared with retraining.
We record the time cost in the unlearning stage for different unlearning baselines and exhibit the results in \cref{fig:efficiency}.
From the experimental results, we can observe that
(1). Negative gradient has the shortest unlearning time and retraining has the longest unlearning time.
(2). Certified unlearning has over 10 times speedup compared with exact unlearning (retrain) over DNNs.
(3). Negative gradient, fine-tuning, and L-CODEC have shorter unlearning time but less desirable unlearning performance than certified unlearning. 
In addition, an important advantage of certified unlearning is its efficiency in adjusting hyperparameters.
In practice, we should choose a proper group of hyperparameters considering the tradeoff between the utility of unlearning target data and the utility of predicting remaining data.
Although the prediction utility can be evaluated efficiently by prevalent utility metrics such as accuracy and F1-score, evaluating unlearning utility is much more costly.
For the unlearning methods without a certification, we can only obtain the unlearning utility by comparing the prediction utility with retrained models or launching membership inference attacks.
In contrast, for certified unlearning methods, the certification budgets can be an effective clue of the unlearning utility and we can adjust the hyperparameters efficiently corresponding to the certification budgets.

\begin{figure*}[t]
    \centering
    \subfigure[Grad Norm]{
    \includegraphics[width=0.25\linewidth]{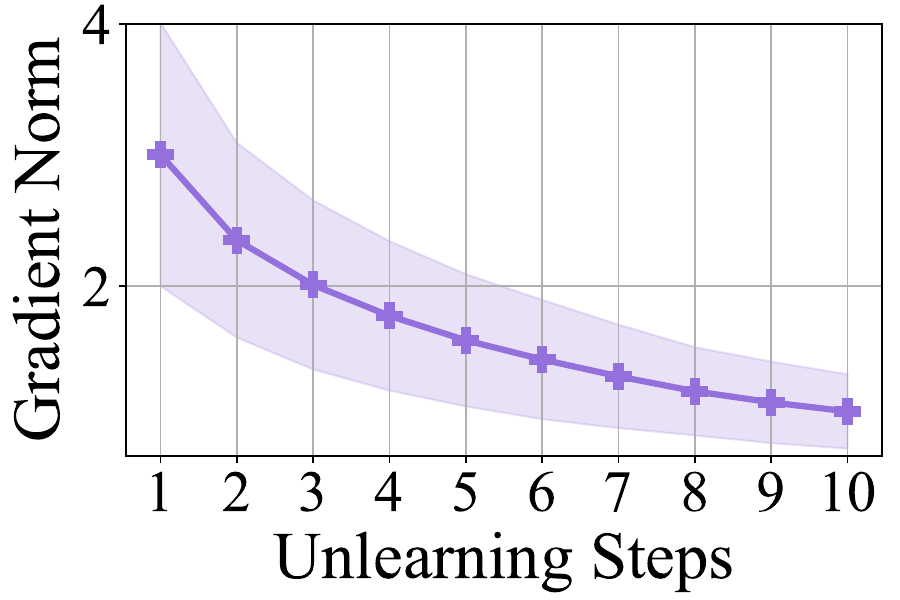}\label{fig:grad}}
    \subfigure[Error Bound]{
    \includegraphics[width=0.25\linewidth]{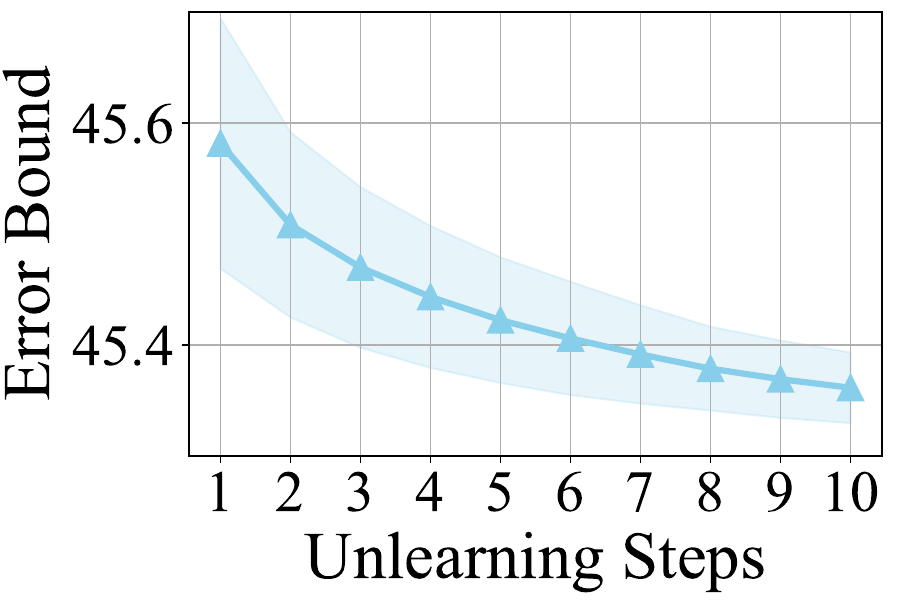}\label{fig:error}}
    \subfigure[Test F1-Score]{
    \includegraphics[width=0.25\linewidth]{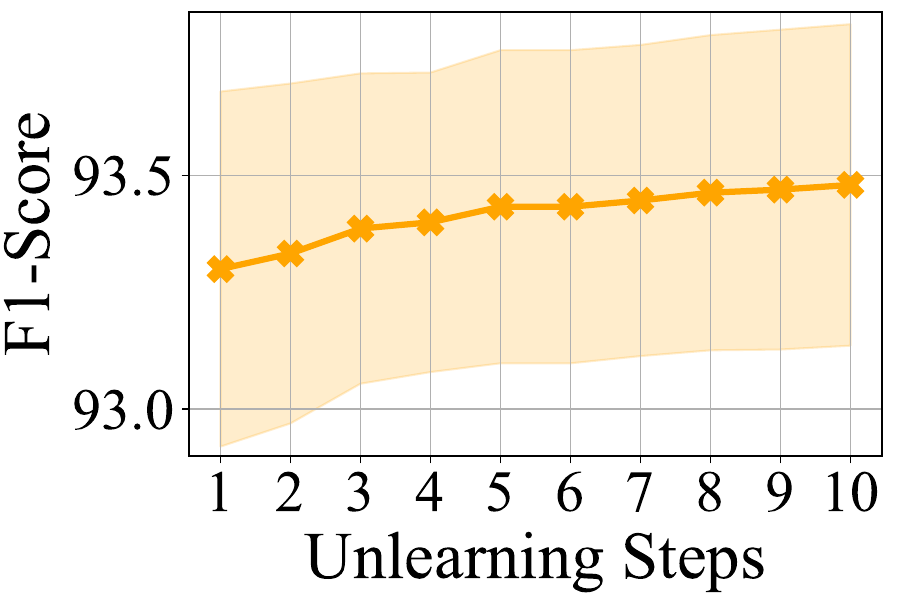}\label{fig:utility}}
    \vspace{-3mm}
    \caption{Gradient norm, approximation error bound, and model utility after each unlearning step.}
    \label{fig:sequential unlearning}
    \vspace{-3mm}
\end{figure*}

\subsection{Ablation Study}
Knowing that certified unlearning for DNNs is practical in real-world experiments, we conduct ablation studies to verify the effectiveness of adopted techniques separately.
Specifically, to verify the effectiveness of the local convex approximation, we remove the local convex approximation by setting $\lambda=0$.
We gradually increase the value of $\lambda$ and record the approximation error and the approximation error bound computed by \cref{eq:practical upper bound}.
Moreover, to verify the effectiveness of the constraint $\|\bm{w}\|_2\leq C$, we remove the constraint and record the approximation error under each $\lambda$.
To compute the approximation error, we compute the Euclidean distance of the unlearned model parameters and the retrained model parameters.
We present the experimental results in \cref{tab:ablation}.
We can observe from \cref{tab:ablation} that
(1). Local convex approximation can distinctly reduce the approximation error bound and slightly reduce the real approximation error as well. 
(2). Certified unlearning with the constraint $\|\bm{w}\|_2\leq C$ has a much lower approximation error (bound).
(3). There exists a potential of the approximation error bound to be further tightened.

In addition, we also remove the inverse Hessian approximation to verify its effect in reducing the time complexity.
In practice, we use \textit{torch.linalg.solve} function to compute the exact inverse matrix of the full Hessian instead of using the inverse Hessian approximation.
Due to the memory limitation, we only conduct the experiment with MLP on the MNIST dataset.
The results show that the inverse Hessian approximation brings \textit{470 times speedup} compared with computing the exact value of the inverse Hessian.

\subsection{Sequential Unlearning}
We conduct experiments on certified unlearning in a sequential setting to verify its feasibility in practice.
In particular, we sequentially delete 10,000 training samples from a trained ResNet-18 model over the SVHN dataset in 10 iterations.
Experimental results are shown in \cref{fig:sequential unlearning}.
Note that each unlearning step can be seen as a single-step Newton update that reduces the loss value over the corresponding retain set. 
In the sequential setting, the estimation in each unlearning step reduces the loss value over the corresponding retain set. 
It is worth noting that the retained set in each step is a subset of the retained set in any previous step. 
Hence, the loss value over the retained set in the $k$-th unlearning step $D_{r_k}$ is reduced in each previous unlearning step. 
Subsequently, the gradient norm $||\nabla\mathcal{L}(\tilde{w}_{k-1},\mathcal{D}_{r_k})||$ also decreases as the model parameters approach a local optimum, so does the approximation error bound according to \Cref{eq:practical upper bound}.
As a result, in a sequential setting, we can still launch the certified unlearning method, while robustly maintaining the utility of the unlearned model.

\subsection{Parameter Study}\label{sec:parameter}
\begin{figure}[t]
    \centering
    \subfigure[Test F1-$\varepsilon$ curve]{
    \includegraphics[width=0.48\linewidth]{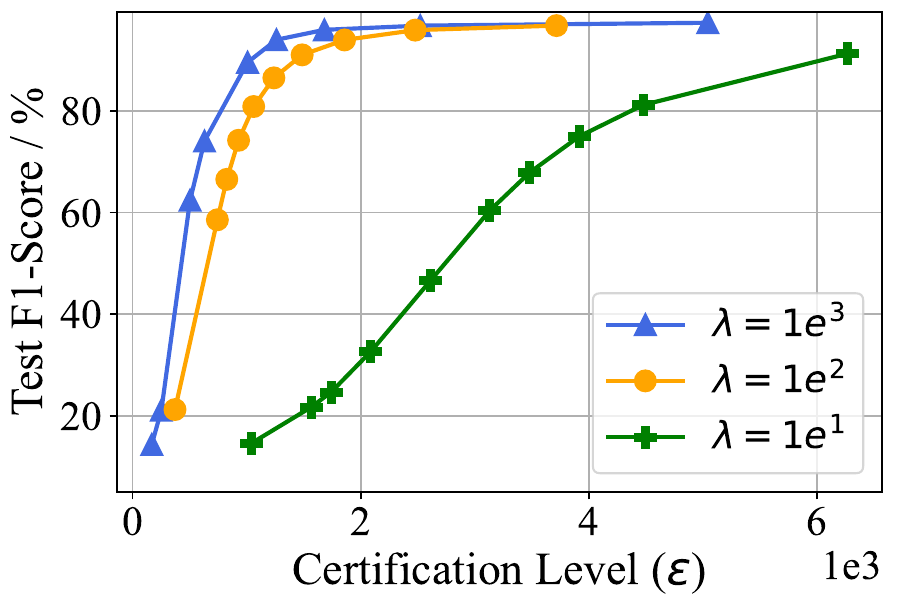}\label{fig:epsilon}}
    \subfigure[Test F1-$\delta$ curve]{
    \includegraphics[width=0.48\linewidth]{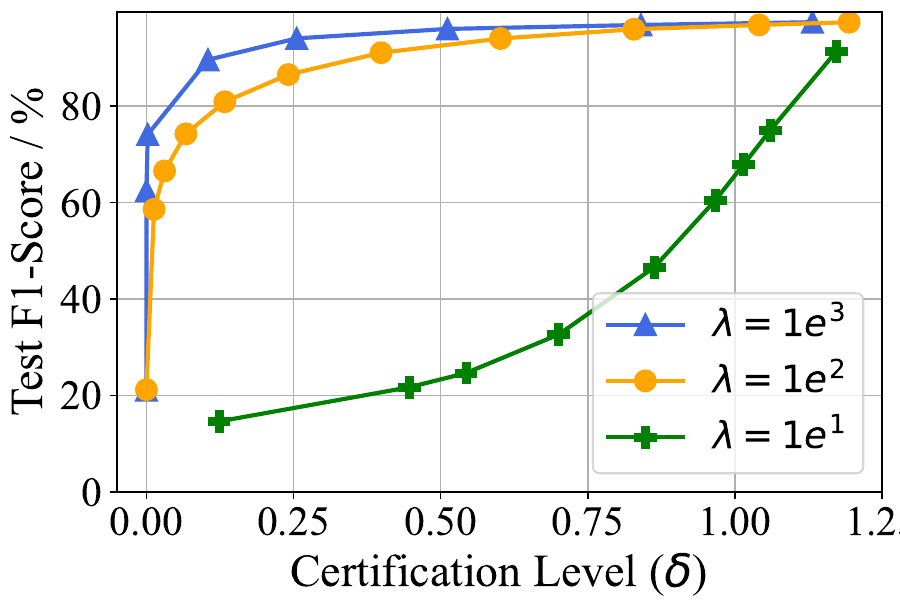}\label{fig:delta}}
    \vspace{-3mm}
    \caption{The effect of local convex coefficient $\lambda$ and certification budget $\varepsilon$ and $\delta$ over the MLP backbone on MNIST.}
    \label{fig:certification budget}
    \vspace{-5mm}
\end{figure}
In this experiment, we discuss the tradeoff between certification budgets ($\varepsilon$, $\delta$) and the utility of the unlearned model.
According to \cref{thm:certified unlearning}, to achieve a tighter certification budget, we have to add a larger noise, which can also lead to utility degradation and limit the practical application of certified unlearning.
To demonstrate the effects of the certification budgets $\varepsilon$ and $\delta$ on the unlearned model utility, we record the micro F1-score of the unlearned model over the test set when adding Gaussian noises with different standard deviations.
We repeat the experiment with different values of the local convex coefficient $\lambda$ and present the results in \cref{fig:certification budget}.
We fix the budget $\delta=0.1$ for \cref{fig:epsilon} and $\varepsilon=1e^3$ for \cref{fig:delta}.
From the experimental results, we can obtain that
(1). The unlearned model utility will decrease as the certification budget becomes tighter.
(2). Under a fixed certification budget, the model utility increases when increasing the value of the local convex coefficient $\lambda$.
(3). When increasing the value of the local convex coefficient $\lambda$, the model utility becomes more sensitive in terms of the certification level.
It is worth noting that the utility increase indicates a smaller noise added in the unlearning stage.
A smaller noise under a fixed certification level further indicates a lower approximation error.
Therefore, \cref{fig:certification budget} also verifies the effectiveness of local convex approximation in mitigating the approximation error bound.
Although the certification budgets in \cref{fig:certification budget} cannot decrease to a very small value when maintaining a desirable utility, our certified unlearning outperforms most unlearning baselines in various unlearning metrics even with a relatively loose budget according to \cref{tab:performance} and \cref{tab:unlearn metric}, which verifies the stringency of certified unlearning. 
Meanwhile, the potential to decrease the certification budget also necessitates further studies on tightening the approximation error bound under nonconvex neural models.

\section{Related Works}
\subsection{Exact Unlearning}
Exact unlearning is a straightforward method for machine unlearning.
The idea of exact unlearning is to find efficient ways to retrain the original model from scratch.
From the beginning of machine unlearning, exact unlearning has been widely studied for statistical query learning~\citep{cao2015towards}, K-Means~\citep{ginart2019making}, random forest~\citep{brophy2021machine}, deep neural networks~\citep{bourtoule2021machine,aldaghri2021coded,kim2022efficient}, and graph neural networks~\citep{chen2022graph}.
Despite having desirable unlearning effectiveness, exact unlearning methods cannot scale to large stochastic models or tackle with batch unlearning settings~\citep{nguyen2022survey}.

\subsection{Certified Unlearning}
Certified unlearning was proposed based on the definition of differential privacy.
The idea of certified unlearning is to find an approximation model similar to the retrained model in distribution by computationally efficient methods.
Existing certified unlearning methods focus primarily on linear models~\citep{guo2020certified,izzo2021approximate,mahadevan2021certifiable}, general convex models~\citep{ullah2021machine,sekhari2021remember,neel2021descent}, Bayesian models~\citep{nguyen2020variational}, sum-product networks~\citep{becker2022certified}, and graph neural networks~\citep{pan2023unlearning,chien2023efficient,wu2023gif,wu2023certified,dong2024idea}.
It is worth noting that most aforementioned certified unlearning methods are based on Newton update, i.e., influence function~\citep{koh2017understanding}, to estimate the retrained model.
Some existing works~\citep{mehta2022deep} took a tentative step towards certified unlearning for DNNs, relying on strong assumptions on the convexity of the objective.

\section{Conclusion}
Despite the extensive study of certified unlearning in many convex machine learning models, the application of certified unlearning in DNNs still poses challenges.
In this paper, we proposed two simple techniques to extend certified unlearning to nonconvex objectives and incorporated an inverse Hessian approximation approach to improve efficiency.
Regarding the real-world scenarios, we also provide the theoretical analysis of the certification under nonconvergence training and sequential unlearning settings.
We conducted extensive empirical experiments to verify the efficacy of our proposed methods and the superiority of certified unlearning in efficiently deleting the information and protecting the privacy of unlearned data.




\section*{Acknowledgements}

This work is supported in part by the National Science Foundation under grants (IIS-2006844, IIS-2144209, IIS-2223769, CNS-2154962, CNS-2213700, and BCS-2228534), the Commonwealth Cyber Initiative Awards under grants (VV-1Q23-007, HV-2Q23-003, and VV-1Q24-011), the JP Morgan Chase Faculty Research Award, and the Cisco Faculty Research Award.

\section*{Impact Statement}

This paper extends the certification of machine unlearning to nonconvex, nonconvergence, and sequential settings in real-world scenarios. 
The goal of this work is to advance the field of Machine Learning. 
There are many potential societal consequences of our work, none of which we feel must be specifically highlighted here.



\bibliography{icml2024}
\bibliographystyle{icml2024}

\newpage
\appendix
\onecolumn
\section{Proofs}\label{sec:proof}

\subsection{Proof of \cref{pro:relationship}}

\begin{proof}
Let the unlearning process be an identical map in terms of the model, i.e., $\mathcal{U}(\mathcal{D},\mathcal{D}_u,\mathcal{A}(\mathcal{D}))=\mathcal{A}(\mathcal{D})$.
Since $\mathcal{A}$ is a differentially private algorithm, we have $\forall\;\mathcal{T}\subseteq\mathcal{H}$,
\begin{equation}
\mathrm{Pr}(\mathcal{A}(\mathcal{D})\in\mathcal{T})\leq e^{\varepsilon}\mathrm{Pr}(\mathcal{A}(\mathcal{D}_r)\in\mathcal{T})+\delta.
\end{equation}
Similarly, we also have $\forall\;\mathcal{T}\subseteq\mathcal{H}$,
\begin{equation}
\mathrm{Pr}(\mathcal{A}(\mathcal{D}_r)\in\mathcal{T})\leq e^{\varepsilon}\mathrm{Pr}(\mathcal{A}(\mathcal{D})\in\mathcal{T})+\delta.
\end{equation}
According to \cref{def:certified unlearning}, $\mathcal{A}(\mathcal{D})$ is an $\varepsilon-\delta$ certified unlearned model under our defined unlearning process $\mathcal{U}$. 
\end{proof}

\subsection{Proof of \cref{thm:certified unlearning}}

\begin{proof}
Refer to the proof of Lemma 10 in~\citet{sekhari2021remember}.
\end{proof}


\subsection{Proof of \cref{lmm:error bound}}

\begin{proof}
First, we have
\begin{equation}\label{eq:substitute w tilde}
\begin{aligned}
\tilde{\bm{w}}-\tilde{\bm{w}}^*&=\bm{w}^*-\bm{H}_{\bm{w}^*}^{-1}\nabla\mathcal{L}(\bm{w}^*,\mathcal{D}_r)-\tilde{\bm{w}}^*.
\end{aligned}
\end{equation}
Considering that $\tilde{\bm{w}}^*=\mathrm{argmin}_{\bm{w}\in\mathcal{H}}\mathcal{L}(\bm{w},\mathcal{D}_r)$, we have $\nabla\mathcal{L}(\tilde{\bm{w}}^*,\mathcal{D}_r)=0$.
By substitute this equation into \cref{eq:substitute w tilde} we have
\begin{equation}\label{eq:variant}
\begin{aligned}
\tilde{\bm{w}}-\tilde{\bm{w}}^*=\bm{w}^*-\tilde{\bm{w}}^*-\bm{H}_{\bm{w}^*}^{-1}\left(\nabla\mathcal{L}(\bm{w}^*,\mathcal{D}_r)-\nabla\mathcal{L}(\tilde{\bm{w}}^*,\mathcal{D}_r)\right).
\end{aligned}
\end{equation}
According to \cref{asp:continuity2}, the loss function $\mathcal{L}$ has the second-order derivative. 
Consequently, according to the fundamental theorem of calculus, we have
\begin{equation}\label{eq:calculus}
\begin{aligned}
\nabla\mathcal{L}(\bm{w}^*,\mathcal{D}_r)-\nabla\mathcal{L}(\tilde{\bm{w}}^*,\mathcal{D}_r)=\int_0^1{\bm{H}_{\bm{w}^*+t(\tilde{\bm{w}}^*-\bm{w}^*)}(\tilde{\bm{w}}^*-\bm{w}^*)}dt.
\end{aligned}
\end{equation}
By incorporate \cref{eq:calculus} into \cref{eq:variant}, we have
\begin{equation}\label{eq:integral}
\begin{aligned}
\tilde{\bm{w}}-\tilde{\bm{w}}^*&=\bm{H}_{\bm{w}^*}^{-1}\left(\bm{H}_{\bm{w}^*}(\bm{w}^*-\tilde{\bm{w}}^*)+\int_0^1{\bm{H}_{\bm{w}^*+t(\tilde{\bm{w}}^*-\bm{w}^*)}(\tilde{\bm{w}}^*-\bm{w}^*)}dt\right) \\
&=\bm{H}_{\bm{w}^*}^{-1}\cdot\int_0^1{(\bm{H}_{\bm{w}^*}-\bm{H}_{\bm{w}^*+t(\tilde{\bm{w}}^*-\bm{w}^*)})(\bm{w}^*-\tilde{\bm{w}}^*)}dt.
\end{aligned}
\end{equation}
We compute the norm of both sides and obtain the right-hand side as
\begin{equation}\label{eq:right hand side}
\begin{aligned}
&\norm{\bm{H}_{\bm{w}^*}^{-1}\cdot\int_0^1(\bm{H}_{\bm{w}^*}-\bm{H}_{\bm{w}^*+t(\tilde{\bm{w}}^*-\bm{w}^*)})(\bm{w}^*-\tilde{\bm{w}}^*)dt} \\
\leq&\|\bm{H}_{\bm{w}^*}^{-1}\|\cdot\int_0^1\|\bm{H}_{\bm{w}^*}-\bm{H}_{\bm{w}^*+t(\tilde{\bm{w}}^*-\bm{w}^*)}\|\cdot\|\bm{w}^*-\tilde{\bm{w}}^*\|dt.
\end{aligned}
\end{equation}
According to \cref{asp:continuity2}, the loss function $l$ has an $M$-Lipschitz Hessian.
Hence, the loss function $\mathcal{L}$ also has an $M$-Lipschitz Hessian.
Consequently, we have
\begin{equation}\label{eq:Lipschitz Hessian}
\begin{aligned}
\|\bm{H}_{\bm{w}^*}-\bm{H}_{\bm{w}^*+t(\tilde{\bm{w}}^*-\bm{w}^*)}\|\leq Mt\|\bm{w}^*-\tilde{\bm{w}}^*\|.
\end{aligned}
\end{equation}
Incorporating \cref{eq:Lipschitz Hessian} into \cref{eq:right hand side} leads to
\begin{equation}\label{eq:rhs final}
\begin{aligned}
&\norm{\bm{H}_{\bm{w}^*}^{-1}\cdot\int_0^1(\bm{H}_{\bm{w}^*}-\bm{H}_{\bm{w}^*+t(\tilde{\bm{w}}^*-\bm{w}^*)})(\bm{w}^*-\tilde{\bm{w}}^*)dt} \\
\leq&\|\bm{H}_{\bm{w}^*}^{-1}\|\cdot\int_0^1Mt\|\bm{w}^*-\tilde{\bm{w}}^*\|^2dt \\
=&\frac{M}{2}\|\bm{H}_{\bm{w}^*}^{-1}\|\cdot\|\bm{w}^*-\tilde{\bm{w}}^*\|^2.
\end{aligned}
\end{equation}
To finish the proof, we can incorporate \cref{eq:rhs final} into \cref{eq:integral} and let the norm be the $\ell$-2 norm.

\end{proof}

\subsection{Proof of \cref{thm:upper bound}}

\begin{proof}
Following the proof of \cref{lmm:error bound}, we have
\begin{equation}\label{eq:variant local}
\begin{aligned}
\tilde{\bm{w}}-\tilde{\bm{w}}^*=\bm{w}^*-\tilde{\bm{w}}^*-(\bm{H}_{\bm{w}^*}+\lambda\bm{I})^{-1}(\nabla\mathcal{L}(\bm{w}^*,\mathcal{D}_r)-\nabla\mathcal{L}\left(\tilde{\bm{w}}^*,\mathcal{D}_r)\right).
\end{aligned}
\end{equation}
Based on \cref{asp:continuity2}, \cref{eq:calculus} still holds.
Incorporating \cref{eq:calculus} into \cref{eq:variant local} leads to
\begin{equation}\label{eq:integral local}
\begin{aligned}
\tilde{\bm{w}}-\tilde{\bm{w}}^*=(\bm{H}_{\bm{w}^*}+\lambda\bm{I})^{-1}\cdot\int_0^1{(\bm{H}_{\bm{w}^*}+\lambda\bm{I}-\bm{H}_{\bm{w}^*+t(\tilde{\bm{w}}^*-\bm{w}^*)})(\bm{w}^*-\tilde{\bm{w}}^*)}dt.
\end{aligned}
\end{equation}
Then, we compute the norm of both sides and obtain the right-hand side as
\begin{equation}\label{eq:rhs local}
\begin{aligned}
&\norm{(\bm{H}_{\bm{w}^*}+\lambda\bm{I})^{-1}\cdot\int_0^1{(\bm{H}_{\bm{w}^*}+\lambda\bm{I}-\bm{H}_{\bm{w}^*+t(\tilde{\bm{w}}^*-\bm{w}^*)})(\bm{w}^*-\tilde{\bm{w}}^*)}dt} \\
\leq&\|(\bm{H}_{\bm{w}^*}+\lambda\bm{I})^{-1}\|\cdot\int_0^1{\left(\lambda+\|\bm{H}_{\bm{w}^*}-\bm{H}_{\bm{w}^*+t(\tilde{\bm{w}}^*-\bm{w}^*)}\|\right)\cdot\|\bm{w}^*-\tilde{\bm{w}}^*\|}dt.
\end{aligned}
\end{equation}
According to \cref{asp:continuity2}, \cref{eq:Lipschitz Hessian} still holds. 
By incorporating \cref{eq:Lipschitz Hessian} into \cref{eq:rhs local}, we have
\begin{equation}\label{eq:rhs final local}
\begin{aligned}
&\norm{(\bm{H}_{\bm{w}^*}+\lambda\bm{I})^{-1}\cdot\int_0^1{(\bm{H}_{\bm{w}^*}+\lambda\bm{I}-\bm{H}_{\bm{w}^*+t(\tilde{\bm{w}}^*-\bm{w}^*)})(\bm{w}^*-\tilde{\bm{w}}^*)}dt} \\
\leq&\|(\bm{H}_{\bm{w}^*}+\lambda\bm{I})^{-1}\|\cdot\int_0^1Mt\|\bm{w}^*-\tilde{\bm{w}}^*\|^2+\lambda\|\bm{w}^*-\tilde{\bm{w}}^*\|dt \\
=&\left(\frac{M}{2}\|\bm{w}^*-\tilde{\bm{w}}^*\|+\lambda\right)\|(\bm{H}_{\bm{w}^*}+\lambda\bm{I})^{-1}\|\cdot\|\bm{w}^*-\tilde{\bm{w}}^*\|.
\end{aligned}
\end{equation}
For the inverse Hessian term, We then have
\begin{equation}\label{eq:inverse norm}
\begin{aligned}
\|(\bm{H}_{\bm{w}^*}+\lambda\bm{I})^{-1}\|_2&=\sqrt{\lambda_{max}[((\bm{H}_{\bm{w}^*}+\lambda\bm{I})^{-1})^\top(\bm{H}_{\bm{w}^*}+\lambda\bm{I})^{-1}]} \\ 
&=\sqrt{\lambda_{max}[((\bm{H}_{\bm{w}^*}+\lambda\bm{I})(\bm{H}_{\bm{w}^*}+\lambda\bm{I})^\top)^{-1}]} \\ 
&=\frac{1}{\sqrt{\lambda_{min}[(\bm{H}_{\bm{w}^*}+\lambda\bm{I})(\bm{H}_{\bm{w}^*}+\lambda\bm{I})^\top]}} \\ 
&=\frac{1}{\lambda_{min}[\bm{H}_{\bm{w}^*}+\lambda\bm{I}]} \\ 
&=\frac{1}{\lambda+\lambda_{min}[\bm{H}_{\bm{w}^*}]},
\end{aligned}
\end{equation}
where $\lambda_{min}[\bm{H}_{\bm{w}^*}]$ and $\lambda_{max}[\bm{H}_{\bm{w}^*}]$ denote the minimum and maximum eigenvalue of the Hessian $\bm{H}_{\bm{w}^*}$.
In addition, we also have $\|\bm{w}^*\|_2\leq C$ and $\|\tilde{\bm{w}}^*\|_2\leq C$.
Consequently, we have
\begin{equation}\label{eq:constraint norm}
\begin{aligned}
\|\bm{w}^*-\tilde{\bm{w}}^*\|_2\leq 2C.
\end{aligned}
\end{equation}
To finish the proof, we can incorporate \cref{eq:inverse norm}, \cref{eq:constraint norm}, and \cref{eq:rhs final local} into \cref{eq:integral local} and let the norm be the $\ell$-2 norm.

\end{proof}


\subsection{Proof of \cref{pro:Hessian estimator}}

\begin{proof}
We compute the expectation value for both sides of \cref{eq:Hessian estimator}.
Considering $\{\bm{H}_{1,\lambda},\dots,\bm{H}_{s,\lambda}\}$ are i.i.d. samples from the Hessian $\bm{H}_{\bm{w}^*}+\lambda\bm{I}$, we have
\begin{equation}\label{eq:exp estimator}
\begin{aligned}
\mathbb{E}[\tilde{\bm{H}}^{-1}_{t,\lambda}]=&\mathbb{E}\left[\bm{I}+\left(\bm{I}-\frac{\bm{H}_{t,\lambda}}{H}\right)\tilde{\bm{H}}^{-1}_{t-1,\lambda}\right] \\
=&\bm{I}+\mathbb{E}[\tilde{\bm{H}}^{-1}_{t-1,\lambda}]-\frac{1}{H}\mathbb{E}[\bm{H}_{t,\lambda}\tilde{\bm{H}}^{-1}_{t-1,\lambda}] \\
=&\bm{I}+\mathbb{E}[\tilde{\bm{H}}^{-1}_{t-1,\lambda}]-\frac{1}{H}\mathbb{E}[\bm{H}_{\bm{w}^*}+\lambda\bm{I}]\mathbb{E}[\tilde{\bm{H}}^{-1}_{t-1,\lambda}].
\end{aligned}
\end{equation}
Let $s\rightarrow\infty$, the limit $\mathbb{E}[\tilde{\bm{H}}^{-1}_{\infty,\lambda}]=\lim_{t\rightarrow\infty}\mathbb{E}[\tilde{\bm{H}}^{-1}_{t,\lambda}]$ exists as $\|\frac{\mathbb{E}[\bm{H}_{t,\lambda}]}{H}\|\leq 1$. We then compute the limit for both sides of \cref{eq:exp estimator} and have
\begin{equation}\label{eq:lim exp estimator}
\begin{aligned}
\mathbb{E}[\tilde{\bm{H}}^{-1}_{\infty,\lambda}]=&\bm{I}+\mathbb{E}[\tilde{\bm{H}}^{-1}_{\infty,\lambda}]-\frac{1}{H}\mathbb{E}[\bm{H}_{\bm{w}^*}+\lambda\bm{I}]\mathbb{E}[\tilde{\bm{H}}^{-1}_{\infty,\lambda}].
\end{aligned}
\end{equation}
Consequently, we have
\begin{equation}
\begin{aligned}
\mathbb{E}\left[\frac{\tilde{\bm{H}}^{-1}_{\infty,\lambda}}{H}\right]=&\mathbb{E}[(\bm{H}_{\bm{w}^*}+\lambda\bm{I})^{-1}].
\end{aligned}
\end{equation}
Hence, we have proved that $\frac{\tilde{\bm{H}}^{-1}_{s,\lambda}}{H}$ is an asymptotic unbiased estimator of the inverse Hessian $(\bm{H}_{\bm{w}^*}+\lambda\bm{I})^{-1}$.

\end{proof}

\subsection{Proof of \cref{thm:efficient upper bound}}

\begin{proof}
Following the proof of \cref{lmm:error bound}, we have
\begin{equation}\label{eq:variant efficient}
\begin{aligned}
\tilde{\bm{w}}-\tilde{\bm{w}}^*&=\bm{w}^*-\tilde{\bm{w}}^*+\frac{n_u}{(n-n_u)H}\tilde{\bm{H}}^{-1}_{s,\lambda}\nabla\mathcal{L}(\bm{w}^*,\mathcal{D}_u) \\
&=\bm{w}^*-\tilde{\bm{w}}^*-\frac{\tilde{\bm{H}}^{-1}_{s,\lambda}}{H}\nabla\mathcal{L}(\bm{w}^*,\mathcal{D}_r) \\
&=\bm{w}^*-\tilde{\bm{w}}^*-\frac{\tilde{\bm{H}}^{-1}_{s,\lambda}}{H}(\nabla\mathcal{L}(\bm{w}^*,\mathcal{D}_r)-\nabla\mathcal{L}\left(\tilde{\bm{w}}^*,\mathcal{D}_r)\right) \\
&=\bm{w}^*-\tilde{\bm{w}}^*-\left((\bm{H}_{\bm{w}^*}+\lambda\bm{I})^{-1}+\frac{\tilde{\bm{H}}^{-1}_{s,\lambda}}{H}-(\bm{H}_{\bm{w}^*}+\lambda\bm{I})^{-1}\right)\cdot\left(\nabla\mathcal{L}(\bm{w}^*,\mathcal{D}_r)-\nabla\mathcal{L}(\tilde{\bm{w}}^*,\mathcal{D}_r)\right).
\end{aligned}
\end{equation}
The right-hand side of \cref{eq:variant efficient} can be divided into two parts, (1). $\bm{w}^*-\tilde{\bm{w}}^*-(\bm{H}_{\bm{w}^*}+\lambda\bm{I})^{-1}(\nabla\mathcal{L}(\bm{w}^*,\mathcal{D}_r)-\nabla\mathcal{L}(\tilde{\bm{w}}^*,\mathcal{D}_r))$ and (2). $((\bm{H}_{\bm{w}^*}+\lambda\bm{I})^{-1}-\tilde{\bm{H}}^{-1}_{s,\lambda}/H)(\nabla\mathcal{L}(\bm{w}^*,\mathcal{D}_r)-\nabla\mathcal{L}(\tilde{\bm{w}}^*,\mathcal{D}_r))$.
We compute the norm for both sides of Eq.~(\ref{eq:variant efficient}), and the norm of the right-hand side is smaller than the summation of the norm of the part (1) and the norm of the part (2).
To find an upper bound of $\|\tilde{\bm{w}}-\tilde{\bm{w}}^*\|$, we can find upper bounds for the norm of part (1) and the norm of part (2).
According to \cref{thm:upper bound}, the norm of part (1) is bounded by $\frac{2C(MC+\lambda)}{\lambda+\lambda_{min}}$.
Next, our goal is to find the upper bound of the norm of part (2).
In particular, the norm of part (2) is smaller than the product $\|(\bm{H}_{\bm{w}^*}+\lambda\bm{I})^{-1}-\tilde{\bm{H}}^{-1}_{s,\lambda}/H\|\cdot\|\nabla\mathcal{L}(\bm{w}^*,\mathcal{D}_r)-\nabla\mathcal{L}(\tilde{\bm{w}}^*,\mathcal{D}_r)\|$.
Refer to Lemma 3.6 in~\citep{agarwal2016second}, and we have
\begin{equation}\label{eq:prob bound}
\norm{(\bm{H}_{\bm{w}^*}+\lambda\bm{I})^{-1}-\frac{\tilde{\bm{H}}^{-1}_{s,\lambda}}{H}}>
\left(16\sqrt{\mathrm{ln}\,d/\rho}\cdot\frac{\lambda+L}{\lambda+\lambda_{min}}+\frac{1}{16}\right),
\end{equation}
with a probability smaller than $\rho$.
According to \cref{asp:continuity1}, we also have
\begin{equation}\label{eq:continuity bound}
\begin{aligned}
\|\nabla\mathcal{L}(\bm{w}^*,\mathcal{D}_r)-\nabla\mathcal{L}(\tilde{\bm{w}}^*,\mathcal{D}_r)\|\leq L\|\bm{w}^*-\tilde{\bm{w}}^*\|\leq 2LC.
\end{aligned}
\end{equation}
Then, we incorporate \cref{eq:prob bound} and \cref{eq:continuity bound} and obtain the upper bound of the norm of part (2): 
$\left(32\sqrt{\mathrm{ln}\,d/\rho}\cdot\frac{\lambda+L}{\lambda+\lambda_{min}}+\frac{1}{8}\right)LC$
To finish the proof, we can combine the upper bounds for the norm of the part (1) and the norm of the part (2) and then obtain \cref{eq:efficient upper bound}.

\end{proof}

\subsection{Proof of \cref{pro:practical upper bound}}

\begin{proof}
Following the proof of \cref{lmm:error bound}, we have
\begin{equation}\label{eq:variant practical}
\begin{aligned}
\tilde{\bm{w}}-\tilde{\bm{w}}^*&=\bm{w}^*-\tilde{\bm{w}}^*+\frac{n_u}{(n-n_u)H}\tilde{\bm{H}}^{-1}_{s,\lambda}\nabla\mathcal{L}(\bm{w}^*,\mathcal{D}_u) \\
&=\bm{w}^*-\tilde{\bm{w}}^*-\frac{\tilde{\bm{H}}^{-1}_{s,\lambda}}{H}\nabla\mathcal{L}(\bm{w}^*,\mathcal{D}_r) \\
&=\bm{w}^*-\tilde{\bm{w}}^*-\frac{\tilde{\bm{H}}^{-1}_{s,\lambda}}{H}\left(\nabla\mathcal{L}(\bm{w}^*,\mathcal{D}_r)-\nabla\mathcal{L}(\tilde{\bm{w}}^*,\mathcal{D}_r)\right)-\frac{\tilde{\bm{H}}^{-1}_{s,\lambda}}{H}\nabla\mathcal{L}(\tilde{\bm{w}}^*,\mathcal{D}_r).
\end{aligned}
\end{equation}
We compute the norm for both sides of \cref{eq:variant practical} and have
\begin{equation}\label{eq:practical inequality}
\begin{aligned}
\|\tilde{\bm{w}}-\tilde{\bm{w}}^*\|\leq\norm{\bm{w}^*-\tilde{\bm{w}}^*-\frac{\tilde{\bm{H}}^{-1}_{s,\lambda}}{H}(\nabla\mathcal{L}(\bm{w}^*,\mathcal{D}_r)-\nabla\mathcal{L}(\tilde{\bm{w}}^*,\mathcal{D}_r))}+\norm{\frac{\tilde{\bm{H}}^{-1}_{s,\lambda}}{H}\nabla\mathcal{L}(\tilde{\bm{w}}^*,\mathcal{D}_r)}.
\end{aligned}
\end{equation}
To find an upper bound of $\|\tilde{\bm{w}}-\tilde{\bm{w}}^*\|$, we can find an upper bound for each norm value on the right-hand side.
According to the proof of \cref{thm:efficient upper bound}, we have $\norm{\bm{w}^*-\tilde{\bm{w}}^*-\frac{\tilde{\bm{H}}^{-1}_{s,\lambda}}{H}(\nabla\mathcal{L}(\bm{w}^*,\mathcal{D}_r)-\nabla\mathcal{L}(\tilde{\bm{w}}^*,\mathcal{D}_r))}$ is upper bounded by $\frac{2C(MC+\lambda)}{\lambda+\lambda_{min}}+\left(32\sqrt{\mathrm{ln}\,d/\rho}\cdot\frac{\lambda+L}{\lambda+\lambda_{min}}+\frac{1}{8}\right)LC$ with a probability larger than $1-\rho$.
Next, our goal is to find the upper bound for $\norm{\frac{\tilde{\bm{H}}^{-1}_{s,\lambda}}{H}\nabla\mathcal{L}(\tilde{\bm{w}}^*,\mathcal{D}_r)}$. 
In particular, we have
\begin{equation}\label{eq:bound term practical}
\begin{aligned}
\norm{\frac{\tilde{\bm{H}}^{-1}_{s,\lambda}}{H}\nabla\mathcal{L}(\tilde{\bm{w}}^*,\mathcal{D}_r)}&\leq\|(\bm{H}_{\bm{w}^*}+\lambda\bm{I})^{-1}\nabla\mathcal{L}(\tilde{\bm{w}}^*,\mathcal{D}_r)\|+\norm{(\frac{\tilde{\bm{H}}^{-1}_{s,\lambda}}{H}-(\bm{H}_{\bm{w}^*}+\lambda\bm{I})^{-1})\nabla\mathcal{L}(\tilde{\bm{w}}^*,\mathcal{D}_r)} \\
&\leq\|(\bm{H}_{\bm{w}^*}+\lambda\bm{I})^{-1}\|\cdot\|\nabla\mathcal{L}(\tilde{\bm{w}}^*,\mathcal{D}_r)\|+\norm{\frac{\tilde{\bm{H}}^{-1}_{s,\lambda}}{H}-(\bm{H}_{\bm{w}^*}+\lambda\bm{I})^{-1}}\cdot\|\nabla\mathcal{L}(\tilde{\bm{w}}^*,\mathcal{D}_r)\|.
\end{aligned}
\end{equation}
According to \cref{eq:inverse norm}, we have $\|(\bm{H}_{\bm{w}^*}+\lambda\bm{I})^{-1}\|\leq\frac{1}{\lambda-\|\bm{H}_{\bm{w}^*}\|}$.
From \cref{eq:prob bound}, we can obtain $\|(\bm{H}_{\bm{w}^*}+\lambda\bm{I})^{-1}-\tilde{\bm{H}}^{-1}_{s,\lambda}/H\|>16\sqrt{\mathrm{ln}\,d/\rho}\cdot\frac{\lambda+L}{\lambda+\lambda_{min}}+\frac{1}{16}$ with a probability smaller than $\rho$.
Additionally, we have $\|\nabla\mathcal{L}(\bm{w}^*,\mathcal{D})\|,\|\nabla\mathcal{L}(\tilde{\bm{w}}^*,\mathcal{D}_r)\|\leq G$.
Incorporating these results into \cref{eq:bound term practical} we have
\begin{equation}\label{eq:final practical bound}
\begin{aligned}
\norm{\frac{\tilde{\bm{H}}^{-1}_{s,\lambda}}{H}\nabla\mathcal{L}(\tilde{\bm{w}}^*,\mathcal{D}_r)}\leq\left(\frac{1}{\lambda-\|\bm{H}_{\bm{w}^*}\|}+16\sqrt{\mathrm{ln}\,d/\rho}\cdot\frac{\lambda+L}{\lambda+\lambda_{min}}+\frac{1}{16}\right)G.
\end{aligned}
\end{equation}
To finish the proof, we incorporate \cref{eq:final practical bound} and results of \cref{thm:efficient upper bound} into \cref{eq:practical inequality}.

\end{proof}

\subsection{Proof of \cref{pro:sequential upper bound}}

\begin{proof}
Follow the proof of \cref{pro:practical upper bound}.
\end{proof}
Intuitively, the error bound will increase as the unlearning algorithm runs for multiple steps since the second term of the error bound $\|\bm{w}^*-\tilde{\bm{w}}^*\|_2^2$ in \Cref{lmm:error bound} tends to increase.
However, we consider the worst-case bound when tackling this term $\|\bm{w}\|_2\leq C\rightarrow\|\bm{w}^*-\tilde{\bm{w}}^*\|_2\leq2C$, which is independent of the unlearning steps.
As a result, the error bound remains unchanged regardless of how large the proportion of the unlearned set is. 
In addition, it means that the error bound can be tightened if we add further assumptions on the distance of the two minimal points $\|\bm{w}^*-\tilde{\bm{w}}^*\|_2$. 

\section{Comparison Between Convex and Nonconvex Objectives}\label{sec:comparison}
In this section, we focus on the results shown in \cref{thm:upper bound}.
Our results proposed in this paper do not require the objective to be convex.
However, if we further assume the objective to be convex, we can obtain a tighter approximation upper bound as follows.
\begin{proposition}\label{pro:convex upper bound}
Let $\bm{w}^*=\mathrm{argmin}_{\|\bm{w}\|_2\leq C}\mathcal{L}(\bm{w},\mathcal{D})$ and $\tilde{\bm{w}}^*=\mathrm{argmin}_{\|\bm{w}\|_2\leq C}\mathcal{L}(\bm{w},\mathcal{D}_r)$.
Let $\tilde{\bm{w}}=\bm{w}^*-\bm{H}_{\bm{w}^*}^{-1}\nabla\mathcal{L}(\bm{w}^*,\mathcal{D}_r)$ be an approximation of $\Tilde{\bm{w}}^*$.
Consider \cref{asp:continuity2} and assume $l(\bm{w},\bm{x})$ to be $K$-strongly convex with respect to $\bm{w}$, then we have
\begin{equation}\label{eq:convex upper bound}
\|\tilde{\bm{w}}-\tilde{\bm{w}}^*\|_2\leq\frac{2MC^2}{K}.
\end{equation}
\end{proposition}
\begin{proof}
As the loss function $l$ is $K$-strongly convex, the loss function $\mathcal{L}$ is also $K$-strongly convex.
Consequently, we have $\|\bm{H}_{\bm{w}^*}\|_2\geq K$ and $\|\bm{H}_{\bm{w}^*}^{-1}\|_2\leq\frac{1}{K}$.
In addition, \cref{eq:constraint norm} still holds.
By incorporating the upper bound $\|\bm{H}_{\bm{w}^*}^{-1}\|_2\leq\frac{1}{K}$ and \cref{eq:constraint norm} into \cref{lmm:error bound}, we can obtain \cref{eq:convex upper bound}.
\end{proof}
Considering the similarity between $\lambda+\lambda_{min}$ in \cref{eq:upper bound} and $K$ in \cref{eq:convex upper bound} (both measuring the objective's convexity), we use $\lambda+\lambda_{min}$ to replace $K$ in \cref{eq:convex upper bound} and derive the approximation error bound for convex models as $\frac{2MC^2}{\lambda+\lambda_{min}}$.
Comparing the result with \cref{eq:convex upper bound}, we find that the Newton update method has a lower approximation error bound for convex models versus nonconvex models.
This result highlights the fact that our method works under \textit{both convex and nonconvex objectives}. 
In particular, our approximation obtains a tighter upper bound for strongly convex objectives.

\section{Bounded Model Parameters is also Necessary for Convex Models}\label{sec:assumption}
In this section, we demonstrate our observation that previous works~\citep{sekhari2021remember} on certified unlearning for convex models implicitly rely on the requirement of bounded model parameters $\|\bm{w}\|_2\leq C$, which means our proposed constraint on model parameters is also suitable for convex models.

Previous works~\citep{sekhari2021remember} studied the certified unlearning and its generalization for convex models.
In particular, they jointly assume the objective to be $L$-Lipschitz continuous and $M$-strongly convex in the proof of certification.
Next, we demonstrate that these two assumptions jointly indicate $\|\bm{w}\|_2\leq C$.
Let $\bm{w}_1,\bm{w}_2$ be two models in the hypothesis space, and $\mathcal{L}$ be the loss function.
Based on the strong convexity, we have
\begin{equation}\label{eq:strong convexity}
\|\nabla\mathcal{L}(\bm{w}_1)-\nabla\mathcal{L}(\bm{w}_2)\|\geq M\|\bm{w}_1-\bm{w}_2\|.
\end{equation}
In addition, according to the Lipschitz continuity, we have
\begin{equation}\label{eq:lipschitz continuity}
\|\nabla\mathcal{L}(\bm{w})\|\leq L,
\end{equation}
for any $\bm{w}$ in the hypothesis space.
Incorporate Eq.~(\ref{eq:lipschitz continuity}) into Eq.~(\ref{eq:strong convexity}) and we have
\begin{equation}
2L\geq\|\nabla\mathcal{L}(\bm{w}_1)-\nabla\mathcal{L}(\bm{w}_2)\|\geq M\|\bm{w}_1-\bm{w}_2\|.
\end{equation}
Consequently, we have $\|\bm{w}_1-\bm{w}_2\|\leq\frac{2L}{M}$ for any $\bm{w}_1$ and $\bm{w}_2$ in the hypothesis space.
The certification requires any two models in the hypothesis space to have a bounded distance, which is a non-trivial condition.
Luckily, by letting $C=\frac{L}{M}$, our constraint on model parameters $\|\bm{w}\|\leq C$ satisfy the condition for certification.
Hence, we find that the assumptions in previous works naturally necessitate a constraint of bounded norm to the model parameters.
In this paper, we exploit projected gradient descent to actively restrict the norm of model parameters for bounding the approximation error.
We argue that this technique is also necessary for fulfilling the assumptions made in previous works to derive the certification for convex models.

\section{Implementation}\label{sec:implementation}
\begin{table}[t]
\centering
\caption{The hyperparameter settings of original models on the corresponding datasets.}
\aboverulesep = 0pt
\belowrulesep = 0pt
\begin{tabular}{lccc}
\toprule
Hyperparameter & MLP & AllCNN & ResNet \\
\midrule
learning rate & $1e^{-3}$ & $1e^{-3}$ & $1e^{-3}$ \\
weight decay & $5e^{-4}$ & $5e^{-4}$ & $5e^{-4}$ \\
epochs & 50 & 50 & 50 \\
dropout & 0.5 & 0.5 & 0.5 \\
batch size & 128 & 128 & 128 \\
param bound $C$ & 10 & 20 & 20 \\
\bottomrule
\end{tabular}
\label{tab:original settings}
\end{table}

We implemented all experiments in the PyTorch~\citep{paszke2019pytorch} package and exploited Adam~\citep{kingma2015adam} as the optimizer for training.
For the training of original models, we exploited Adam as the optimizer.
We set the learning rate as $1e^{-3}$, the weight decay parameter as $5e^{-4}$, and the training epochs number as $50$.
We ran all experiments on an Nvidia RTX A6000 GPU.
All experiments are conducted based on three real-world datasets: MNIST~\citep{lecun1998gradient}, CIFAR-10~\citep{krizhevsky2009learning}, and SVHN~\citep{netzer2011reading}.
All datasets are publicly accessible (MNIST with GNU General Public License, CIFAR-10 with MIT License, and SVHN with CC BY-NC License).
We reported the average value and the standard deviation of the numerical results under three different random seeds.
For the relearn time in \Cref{tab:unlearn metric}, we directly report the rounded mean value without the standard deviation as the value of the epoch number is supposed to be an integer.
The unlearned data is selected randomly from the training set.
Detailed hyperparameter settings of the original models are presented in Table~\ref{tab:original settings}.
The hyperparameter settings of unlearning baselines are shown as follows.
\begin{itemize}[leftmargin=*]
\item \textbf{Retrain from scratch}: size of unlearned set: 1,000.
\item \textbf{Fine tune}: size of unlearned set: 1,000; learning rate: $1e^{-3}$; epochs: 1.
\item \textbf{Negative gradient}: size of unlearned set: 1,000; learning rate: $1e^{-4}$; epochs: 1.
\item \textbf{Fisher forgetting}: size of unlearned set: 1,000; $\alpha$: $1e^{-6}$ for MLP, $1e^{-8}$ for AllCNN and ResNet.
\item \textbf{L-CODEC}: size of unlearned set: 1,000; number of perturbations: 25; Hessian type: Sekhari; $\varepsilon$: 100; $\delta$: 0.1; $\ell$-2 regularization: $5e^{-4}$.
\item \textbf{Certified unlearning}: size of unlearned set: 1,000; number of recursion $s$: 1,000; standard deviation $\sigma$: $1e^{-2}$ for MLP, $1e^{-3}$ for AllCNN and ResNet; continuity coefficients $L$: 1, $M$: 1; minimal eigenvalue of Hessian $\lambda_{min}$: 0; convex coefficient $\lambda$: 1 for MLP, 200 for AllCNN, 2,000 for ResNet; Hessian scale $H$: 10 for MLP, 20,000 for AllCNN, 50,000 for ResNet.
\end{itemize}
The detailed meanings of hyperparameters in unlearning baselines can be found in the original papers.
For $\lambda$, since finding a practical upper bound of the norm of Hessian can be intractable for real-world tasks, we follow most previous works~\citep{koh2017understanding,wu2023gif,wu2023certified} to set it as a hyperparameter which can be chosen flexibly to adapt to different scenarios. 
In addition, we compute the norm of the Hessian in the case of MLP over MNIST. 
To reduce the time complexity, we use the Hessian with respect to a single random mini-batch of $D_r$ as an unbiased estimation of the full Hessian. 
The results of the mean value and the variance of the norm of the stochastic Hessian (under 10 different random seeds) is $12.11 \pm 0.63$, which falls into the range of $\lambda$ in \Cref{tab:ablation}. 
When $\lambda<12.11$, the computed error bound is still valid (from \Cref{tab:ablation}, as the value of $\lambda$ decreases, the computed error bound increases correspondingly, indicating that a smaller choice of $\lambda$ can still lead to a valid but larger error bound). 
In this case, the certification requires adding a larger noise to hide the (overestimated) remaining information of the unlearned data.
Different from $\lambda$, the values of $L$, $M$, and $\lambda_{min}$ only affect the value of the approximation error bound $\Delta$.
In practice, following previous works~\citep{guo2020certified}, we first determine the variance of noise $\sigma$ (to preserve the model utility after adding noise) and then obtain the certification level $\varepsilon$ and $\delta$ which can be achieved.
Hence, the certification holds for any choices of $L$, $M$, and $\lambda_{min}$, but the calculated certification level can be imprecise in some cases.
To weaken the dependency of the approximation error bound on the Lipschitz constants, we leave more general theoretical results as future works.
It is also worth noting that the values of certification budgets $\varepsilon$ and $\delta$ are flexible in our experiments. 
In particular, we fixed the noise and the budgets $\varepsilon$ and $\delta$ can be tuned in a range for specific needs (we have $\varepsilon\in[1e^2,5e^3]$ when $\delta\in[0.1,1]$, and the decrease of one can lead to the increase of the other).
Detailed results regarding the certification budgets are shown in the parameter study in \cref{sec:parameter}.
In addition, we list some key packages in Python required for implementing certified unlearning as follows.
\begin{itemize}[leftmargin=*]
\item python == 3.9.16
\item torch == 1.12.1
\item torchvision == 0.13.1
\item numpy == 1.24.3
\item scikit-learn == 1.2.2
\item scipy == 1.10.1
\end{itemize}

\section{Additional Experiments}

\subsection{Effect of Bounded Model Parameters}
In our experiments, we adopt PGD during the training of the original model and the retrained model.
PGD is conducted solely for ensuring the strictness of our theoretical results while trying not to affect the model utility. 
To achieve this, we deliberately choose the value of $C$ in the constraint $|w|_2<C$ to make PGD have little impact on the model utility. 
After verifying with experiments, we find that $C=10$ for MLP over MNIST, $C=20$ for CNN over CIFAR-10, and $C=20$ for ResNet over SVHN can be a desirable choice. 
In this experiment, we record and report the utility metrics (micro F1-score over the test set) of the original model and retrained model without conducting PGD in \Cref{tab:pgd_effect}.
Compared with the corresponding results in \Cref{tab:performance}, we can verify that our adopted PGD modification with a carefully chosen constraint does not affect the model utility distinctly.

\begin{table}[t]
\centering
\vspace{-3mm}
\caption{The micro F1-score of the original model and the retrained model without using PGD.}
\aboverulesep = 0pt
\belowrulesep = 0pt
\begin{tabular}{lccc}
\toprule
 & MLP \& MNIST & CNN \& CIFAR-10 & ResNet \& SVHN \\
\midrule
Original & 97.03 $\pm$ 0.25 & 83.16 $\pm$ 0.62 & 93.92 $\pm$ 0.19 \\
Retrain & 96.80 $\pm$ 0.20 & 83.39 $\pm$ 0.59 & 93.80 $\pm$ 0.31 \\
\bottomrule
\end{tabular}
\label{tab:pgd_effect}
\end{table}

\begin{table}[t]
\small
\renewcommand{\arraystretch}{1.05}
\tabcolsep = 3pt
\centering
\vspace{-3mm}
\caption{Comparison of accuracy among three advanced unlearning baselines over three popular DNNs across three real-world datasets. We record the micro F1-score of the predictions on the unlearned set $\mathcal{D}_u$, retained set $\mathcal{D}_r$, and test set $\mathcal{D}_t$.}
\label{tab:additional_experiment}
\vspace{1mm}
\aboverulesep = 0pt
\belowrulesep = 0pt
\begin{tabular}{c|ccc|ccc|ccc}
\toprule
\multirow{2}{*}{\textbf{Method}} & \multicolumn{3}{c|}{\textbf{MLP \& MNIST}} & \multicolumn{3}{c|}{\textbf{AllCNN \& CIFAR-10}} & \multicolumn{3}{c}{\textbf{ResNet18 \& SVHN}} \\
& F1 on $\mathcal{D}_u$ & F1 on $\mathcal{D}_r$ & F1 on $\mathcal{D}_t$ & F1 on $\mathcal{D}_u$ & F1 on $\mathcal{D}_r$ & F1 on $\mathcal{D}_t$ & F1 on $\mathcal{D}_u$ & F1 on $\mathcal{D}_r$ & F1 on $\mathcal{D}_t$ \\
\midrule
SCRUB & 98.10 {\scriptsize$\pm$ 0.60} & 98.12 {\scriptsize$\pm$ 0.05} & 97.25 {\scriptsize$\pm$ 0.04} & 90.73 {\scriptsize$\pm$ 2.63} & 96.83 {\scriptsize$\pm$ 0.82} & 86.31 {\scriptsize$\pm$ 0.31} & 94.07 {\scriptsize$\pm$ 0.38} & 96.09 {\scriptsize$\pm$ 1.25} & 94.02 {\scriptsize$\pm$ 0.38} \\
NegGrad+ & 97.43 {\scriptsize$\pm$ 0.60} & 98.92 {\scriptsize$\pm$ 0.01} & 97.68 {\scriptsize$\pm$ 0.07} & 80.70 {\scriptsize$\pm$ 8.79} & 96.58 {\scriptsize$\pm$ 0.47} & 86.05 {\scriptsize$\pm$ 1.52} & 91.03 {\scriptsize$\pm$ 1.32} & 97.58 {\scriptsize$\pm$ 0.45} & 95.07 {\scriptsize$\pm$ 0.12} \\
$\ell$-1 Sparse & 97.83 {\scriptsize$\pm$ 0.47} & 98.40 {\scriptsize$\pm$ 0.19} & 97.24 {\scriptsize$\pm$ 0.12} & 87.17 {\scriptsize$\pm$ 2.80} & 92.26 {\scriptsize$\pm$ 0.37} & 83.99 {\scriptsize$\pm$ 0.32} & 94.33  {\scriptsize$\pm$ 0.61} & 95.55 {\scriptsize$\pm$ 0.33} & 93.67 {\scriptsize$\pm$ 0.16} \\
\bottomrule
\end{tabular}
\end{table}

\subsection{Comparison with Advanced Baselines}
To strengthen the soundness of certified unlearning, we supply three more advanced baselines for approximate unlearning: SCRUB~\citep{kurmanji2023towards}, NegGrad+~\citep{liu2023model}, and $\ell$-1 Sparse~\citep{liu2023model}. 
For SCRUB, we adopt the provided hyperparameter settings in the code base (small scale for MLP and large scale for CNN and ResNet). 
For NegGrad+, we tune the hyperparameter manually and find the optimal setting $\alpha$=0.95, epoch=1, and learning rate=1e-4. 
For $\ell$-1 Sparse~\citep{liu2023model}, the proposed framework consists of two parts, pruning and unlearning. 
The pruning step is orthogonal to all other unlearning baselines. 
Hence, we drop the pruning step for all baselines to ensure a fair comparison (similar to the case that we use PGD during training for all baselines).
It is worth noting that other than pruning, $\ell$-1 sparse unlearning proposes a novel unlearning technique as well (adding an $\ell$-1 penalty to the objective). 
As a result, we include $\ell$-1 sparse unlearning without pruning in the additional results shown in \Cref{tab:additional_experiment}.
Comparing the additional results with \Cref{tab:performance}, we can observe that our proposed certified method still has the most desirable and robust performance. 
Compared with the negative gradient method, NegGrad+ improves the utility over the retain and test sets but still has a lower performance on the forget set.
With the $\ell$-1 penalty, the $\ell$-1 sparse unlearning has a distinct improvement compared with vanilla fine-tuning. 


\end{document}